\documentclass[oneside,11pt]{article}

\usepackage{amsmath,amsthm,amsopn,amsfonts,pdfpages,url}
\usepackage{kbordermatrix}
\usepackage{graphics,relsize}
\usepackage{subfig}
\usepackage{bbm}
\usepackage{natbib}
\bibpunct{(}{)}{;}{a}{,}{,}

\newtheorem{theorem}{Theorem}
\newtheorem{definition}[theorem]{Definition}
\newtheorem{proposition}[theorem]{Proposition}

\newtheorem{remark}[theorem]{Remark}

\newcommand{\R}{\mathbb{R}}
\newcommand{\fu}{\mathfrak{u}}
\newcommand{\fe}{\mathfrak{e}}
\newcommand{\tfu}{\mathfrak{\tilde u}}
\newcommand{\nvec}{\vec{n}}
\newcommand{\nhat}{\hat{n}}
\newcommand{\pivec}{\vec{\pi}}
\newcommand{\tr}{\operatorname{tr}}
\newcommand{\p}{\mathbb{P}}
\newcommand{\e}{\mathbb{E}}
\newcommand{\AP}{\operatorname{AP}}
\newcommand{\ML}{\operatorname{ML}}
\newcommand{\SBM}{\operatorname{SBM}}
\newcommand{\ExpSBM}{\operatorname{ExpSBM}}
\newcommand{\Norm}{\operatorname{Norm}}
\newcommand{\calL}{\mathcal{L}}
\newcommand{\Lspec}{\mathcal{L}^{\operatorname{SP}}}
\newcommand{\Lcan}{\mathcal{L}^{\operatorname{C}}}
\newcommand{\Lml}{\mathcal{L}^{\ML}}
\newcommand{\LmlRES}{\mathcal{L}^{\ML}_R}
\newcommand{\LFEML}{\mathcal{L}^{\operatorname{ML}}_F}
\newcommand{\bhat}{\hat{b}}
\newcommand{\Phat}{\hat{P}}
\newcommand{\Bhat}{\widehat{B}}
\newcommand{\muhat}{\hat{\mu}}
\newcommand{\etatilde}{\tilde{\eta}}
\newcommand{\xitilde}{\tilde{\xi}}
\newcommand{\etaF}{\eta_F}
\newcommand{\xiF}{\xi_F}
\newcommand{\calEone}{\mathcal{E}^{(1)}}
\newcommand{\calEtwo}{\mathcal{E}^{(2)}}
\newcommand{\calG}{\mathcal{G}}
\newcommand{\epsilonout}[1]{\epsilon_{#1,\bullet}}
\newcommand{\tepsilonout}[1]{\tilde\epsilon_{#1,\bullet}}

\newcommand{\LL}[2]{\log\Bigg(\frac{\Lambda_{b(#1),b(#2)} }{1-\Lambda_{b(#1),b(#2)} } \Bigg) }

\begin{document}

\title{On the Consistency of the Likelihood Maximization Vertex Nomination Scheme:  Bridging the Gap Between Maximum Likelihood Estimation and Graph Matching}

\author{Vince Lyzinski$^*$, Keith Levin$^\dagger$, Donniell E. Fishkind$^\ddagger$, Carey E. Priebe$^{\ddagger}$
\vspace{3mm}\\
$^*$Human Language Technology Center of Excellence, Johns Hopkins University \\
$^\dagger$Department of Computer Science, Johns Hopkins University\\
$^\ddagger$Department of Applied Mathematics and Statistics, Johns Hopkins University 
}

\maketitle

\begin{abstract}
Given a graph in which a few vertices are deemed interesting a priori, the vertex nomination task is to order the remaining vertices into a nomination list such that there is a concentration of interesting vertices at the top of the list. Previous work has yielded several approaches to this problem, with theoretical results in the setting where the graph is drawn from a stochastic block model (SBM), including a vertex nomination analogue of the Bayes optimal classifier. In this paper, we prove that maximum likelihood (ML)-based vertex nomination is consistent, in the sense that the performance of the ML-based scheme asymptotically matches that of the Bayes optimal scheme. We prove theorems of this form both when model parameters are known and unknown. Additionally, we introduce and prove consistency of a related, more scalable restricted-focus ML vertex nomination scheme. Finally, we incorporate vertex and edge features into ML-based vertex nomination and briefly explore the empirical effectiveness of this approach.
\end{abstract}

\section{Introduction and Background}
\label{S:intro}

Graphs are a common data modality, useful for modeling complex relationships between objects, with applications spanning fields as varied as biology~\citep{jeong2001lethality,bullmore2009complex}, sociology~\citep{wasserman}, and
computer vision \citep{foggia2014graph,kandel2007applied}, to name a few.
For example, in neuroscience, vertices may be neurons and edges adjoin
pairs of neurons that share a synapse~\citep{bullmore2009complex};
in social networks, vertices may correspond to people and edges to friendships
between them~\citep{carrington2005models,yang2015defining};
in computer vision, vertices may represent pixels in an image and edges may represent spatial proximity or multi-resolution mappings~\citep{kandel2007applied}.
In many useful networks, vertices with similar attributes form densely-connected communities compared to vertices with highly disparate attributes, and uncovering these communities is an important step in understanding the structure of the network.
There is an extensive literature devoted to uncovering this community structure in network data, including methods based on maximum modularity~\citep{Newman2004,newman2006modularity}, spectral partitioning algorithms~\citep{von2007tutorial,rohe2011spectral,STFP,perfect}, and likelihood-based methods \citep{BicChe2009}, among others.

In the setting of {\it vertex nomination}, one community in the network is of particular interest, and the inference task is to order the vertices into a nomination list with those vertices from the community of interest concentrating at the top of the list.
See~\cite{marchette2011vertex,coppersmith2012vertex,coppersmith2014vertex,fishkind2015vertex} and the references contained therein for a review of the relevant vertex nomination literature.
Vertex nomination is a semi-supervised inference task, with example vertices from the community of interest---and, ideally, also examples not from the community of interest---being leveraged in order to create a nomination list.
In this way, the vertex nomination problem is similar to the problem faced by personalized recommender systems \citep[see, for example,][]{resnick1997recommender,ricci2011introduction},
where, given a training list of objects of interest, the goal is to arrange the remaining objects into a recommendation list with ``interesting'' objects concentrated at the top of the list.
The main difference between the two inference tasks is that in vertex nomination the features of the data are encoded into the topology of a network, rather than being observed directly as features (though see Section~\ref{S:gen}
for the case where vertices are annotated with additional
information in the form of features).

In this paper, we develop the notion of a consistent vertex nomination scheme (Definition \ref{def:consis}).
We then proceed to prove that the maximum likelihood vertex nomination scheme of \cite{fishkind2015vertex} is consistent under mild model assumptions on the underlying stochastic block model (Theorem \ref{thm:mlconsis}).
In the process, we propose a new, efficiently exactly solvable likelihood-based vertex nomination scheme, the restricted-focus maximum likelihood vertex nomination scheme, $\LmlRES$, and prove the analogous consistency result (Theorem \ref{thm:restmlconsis}).
In addition, under mild model assumptions, we prove that both schemes maintain their consistency when the stochastic block model parameters are unknown and are estimated using the seed vertices (Theorems \ref{thm:hatconsistency} and \ref{thm:rhatconsistency}).
In both cases, we show that consistency is possible even when the seeds are an asymptotically vanishing portion of the graph.
Lastly, we show how both schemes can be easily modified to incorporate edge weights and vertex features (Section \ref{S:gen}), before demonstrating the practical effect of our theoretical results on real and synthetic data (Section \ref{sec:experiments}) and closing with a brief discussion (Section~\ref{sec:discussion}).

\noindent{\bf Notation:}
We say that a sequence of random variables $(X_n)_{n=1}^\infty$
converges almost surely to random variable $X$,
written $X_n \rightarrow X$ a.s.,
if $\p[ \lim_{n \rightarrow \infty} X_n = X ] = 1$.
We say a sequence of events $(A_n)_{n=1}^\infty$ occurs
almost always almost surely (abbreviated a.a.a.s.) if
with probability 1, $A_n^c$ occurs for at most finitely many $n$.
By the Borel-Cantelli lemma, $\sum_{n=1}^\infty \p[ A_n^c ] < \infty$
implies $(A_n)_{n=1}^\infty$ occurs a.a.a.s.
We write $\calG_n$ to denote the set of all (possibly weighted)
graphs on $n$ vertices.
Throughout, without loss of generality, we will assume that
the vertex set is given by $V = \{1,2,\dots,n\}$.
For a positive integer $K$, we will often use $[K]$ to denote the
set $\{1,2,\dots,K\}$.
For a set $V$, we will use $\binom{V}{2}$ to denote the set of all pairs
of distinct elements of $V$.
That is, $\binom{V}{2} = \{ \{u,v\} : u,v \in V, u \neq v \}$.
For a function $f$ with domain $V$, we write $f_{|_U}$ to denote the
restriction of $f$ to the set $U \subset V$.

\subsection{Background}

Stochastic block model random graphs offer a theoretically tractable model for graphs with latent community structure \citep{rohe2011spectral,STFP, BicChe2009}, and have been widely used in the literature to model community structure in real networks~\citep{Airoldi2008,karrer11:_stoch}.
While stochastic block models can be too simplistic to capture the eccentricities of many real graphs, they have proven to be a useful, tractable surrogate for more complicated networks~\citep{airoldi13:_stoch,olhede_wolfe_histogram}.
\begin{definition}
\label{def:SBM}
Let $K$ and $n$ be positive integers and let
$\nvec = (n_1,n_2,\dots,n_K)^\top \in \R^K$ be a
vector of positive integers with $\sum_k n_k = n$.
Let $b : [n] \rightarrow [K]$ and let $\Lambda \in [0,1]^{K \times K}$
be symmetric.
A $\calG_n$-valued random graph $G$ is an instantiation of a $(K,\nvec,b, \Lambda)$ conditional Stochastic Block Model, written $G\sim \SBM(K,\nvec,b, \Lambda)$, if
\begin{itemize}
\item[i.] The vertex set $V$ is partitioned into $K$ blocks, $V_1,V_2,\dots,V_K$ of cardinalities $|V_k| = n_k$ for $k=1,2,\dots,K$;
\item[ii.]  The block membership function $b:V\rightarrow[K]$ is such that for each $v\in V$, $v\in V_{b(v)}$;
\item[iii.] The symmetric block communication matrix $\Lambda\in[0,1]^{K\times K}$ is such that for each $\{v,u\}\in\binom{V}{2}$, there is an edge between vertices $u$ and $v$ with probability $\Lambda_{b(u),b(v)}$, independently of all other edges.
\end{itemize}
\end{definition}
\noindent

Without loss of generality, let $V_1$ be the block of interest for vertex nomination.
For each $k\in[K]$, we further decompose $V_k$ into $V_k=S_k\cup U_k$ (with $|S_k|=m_k$), where the vertices in $S:=\cup_{k}S_k$ have their block membership observed {\it a priori}.
We call the vertices in $S$ {\it seed vertices}, and let $m=|S|$.
We will denote the set of nonseed vertices by $U=\cup_k U_k$, and for all $k\in[K]$, let $\fu_k:=n_k-m_k=|U_k|$ and $n-m=\mathfrak{u}=|U|.$
Throughout this paper, we assume that the seed vertices $S$ are chosen
uniformly at random from all possible subsets of $V$ of size $m$.
The task in vertex nomination is to leverage the information contained in the seed vertices to produce a \emph{nomination list} $\mathcal{L}:U\rightarrow [\fu]$ (i.e., an ordering of the vertices in $U$) such that the vertices in $U_1$ concentrate at the top of the list.
We note that, strictly speaking, a nomination list $\calL$ is also
a function of the observed graph $G$, a fact that we suppress for
ease of notation.
We measure the efficacy of a nomination scheme via \emph{average precision}
\begin{equation} \label{eq:def:AP}
\AP(\mathcal{L})
= \frac{1}{\fu_1}\sum_{i=1}^{\fu_1}\frac{\sum_{j=1}^i\mathbbm{1}\{\mathcal{L}^{-1}(j)\in U_1  \}}{i}.
\end{equation}
AP ranges from $0$ to $1$, with a higher value indicating a more effective nomination scheme:  indeed, $\AP(\mathcal{L})=1$ indicates that the first $\fu_1$ vertices in the nomination list are all from the block of interest, and $\AP(\calL)=0$ indicates that none of the $\fu_1$ top-ranked vertices are from the block of interest.
Letting $H_k = \sum_{j=1}^k 1/j$ denote the $k$-th harmonic number,
with the convention that $H_0 = 0$, we can rearrange~\eqref{eq:def:AP} as
\begin{equation*}
\AP(\mathcal{L})
= \sum_{i=1}^{\fu_1} \frac{ H_{\fu_1} - H_{i-1} }{ \fu_1 } \mathbbm{1}\{\mathcal{L}^{-1}(i)\in U_1  \},
\end{equation*}
from which we see that the average precision is simply a convex combination of the
indicators of correctness in the rank list,
in which correctly placing an interesting vertex higher in the
nomination list (i.e., with rank close to 1) is rewarded more than
correctly placing an interesting vertex lower in the nomination list.

In~\cite{fishkind2015vertex}, three vertex nomination schemes are presented in the context of stochastic block model random graphs:
the canonical vertex nomination scheme, $\Lcan$, which is suitable for small graphs (tens of vertices); the likelihood maximization vertex nomination scheme, $\Lml$, which is suitable for small to medium graphs (up to thousands of vertices); and the spectral partitioning vertex nomination scheme, $\Lspec$, which is suitable for medium to very large graphs (up to tens of millions of vertices).
In the stochastic block model setting,
the canonical vertex nomination scheme
is provably optimal: under mild model assumptions, $\e \AP(\Lcan)\geq \e \AP(\mathcal{L})$ for any vertex nomination scheme $\mathcal{L}$~\citep{fishkind2015vertex}, where the expectation is with respect to a $\calG_{m+n}$-valued
random graph $G$ and the selection of the seed vertices.
Thus, the canonical method is the vertex nomination
analogue of the Bayes classifier, and this motivates
the following definition:
\begin{definition}
\label{def:consis}
Let $G\sim \SBM(K,\nvec,b, \Lambda)$.  With notation as above, a vertex nomination scheme $\mathcal{L}$ is consistent if
$$\lim_{n\rightarrow\infty}|\e \AP(\mathcal{L}^C)-\e \AP(\mathcal{L})|=0. $$
\end{definition}
\noindent In our proofs below, where we establish the consistency of
two nomination schemes, we prove a stronger fact,
namely that $\AP( \calL ) = 1$ a.a.a.s.
We prefer the definition of consistency given in
Definition~\ref{def:consis} since it allows us to speak about the best
possible nomination scheme even when the model is such that
$\lim_{n \rightarrow \infty} \e \AP(\Lcan) < 1$.

In \cite{fishkind2015vertex}, it was proven that under mild
assumptions on the stochastic block model underlying~$G$, we have
$$\lim_{n\rightarrow\infty}\e \AP(\Lspec)=1,$$
from which the consistency of $\Lspec$ follows immediately.
The spectral nomination scheme $\Lspec$ proceeds by first $K$-means clustering the adjacency spectral embedding \citep{STFP} of $G$, and then nominating vertices based on their distance to the cluster of interest.
Consistency of $\Lspec$ is an immediate consequence of the fact that, under mild model assumptions on the underlying stochastic block model, $K$-means clustering of the adjacency spectral embedding of $G$ perfectly clusters the vertices of $G$ a.a.a.s. \citep{perfect}.

\cite{BicChe2009} proved that maximum likelihood estimation provides consistent estimates of the model parameters in a more common variant of the conditional stochastic block model of Definition~\ref{def:SBM}, namely, in the stochastic block model with random block assignments:
\begin{definition}
\label{def:SBM2}
Let $K,n$ and $\Lambda$ be as above.
Let $\pivec = \pi_1,\pi_2,\ldots,\pi_K)^\top \in \Delta^{K-1}$ be
a probability vector over $K$ outcomes and let $\tau : V \rightarrow [K]$
be a random function.
A $\calG_n$-valued random graph $G$ is an instantiation of a $(K,\pivec,\tau,\Lambda)$ Stochastic Block Model with random block assignments, written $G\sim \SBM(K,\pivec, \tau,\Lambda)$, if
\begin{itemize}
\item[i.] For each vertex $v \in V$ and block $k \in [K]$, independently of all other vertices, the block assignment function $\tau:V\rightarrow[K]$ assigns $v$ to block $k$ with probability $\pi_k$ (i.e., $\p[ \tau(v) = k ] = \pi_k$);
\item[ii.] The symmetric block communication matrix $\Lambda\in[0,1]^{K\times K}$ is such that, conditioned on $\tau$, for each $\{v,u\}\in\binom{V}{2}$ there is an edge between vertices $u$ and $v$ with probability $\Lambda_{\tau(u),\tau(v)}$, independently of all other edges.
\end{itemize}
\end{definition}
\noindent A consequence of the result of \cite{BicChe2009} is that the maximum likelihood estimate of the block assignment function perfectly clusters the vertices a.a.a.s.
in the setting where $G \sim \SBM(K,\pivec, \tau,\Lambda)$.
This bears noting, as our maximum likelihood vertex nomination schemes
$\Lml$ and $\LmlRES$ (defined below in Section~\ref{S:SGM})
proceed by first constructing a maximum likelihood estimate of
the block membership function $b$,
then ranking vertices based on a measure of model misspecification.
Extending the results from \cite{BicChe2009} to our present framework---where we consider $\Lambda$ and $\vec n$ to be known (or errorfully estimated via seeded vertices) as opposed to parameters to be optimized over in the likelihood function as done in \cite{BicChe2009}---is not immediate.

We note the recent result by~\cite{Newman2016}, which shows the
equivalence of maximum-likelihood and maximum modularity methods in
a special case of the stochastic block model when $\Lambda$ is known.
Our results, along with this recent result,
immediately imply a consistent maximum modularity-based
vertex nomination scheme under that special-case model.

\section{Graph Matching and Maximum Likelihood Estimation}
\label{S:SGM}

Consider $G\sim \SBM(K,\nvec,b, \Lambda)$ with associated adjacency matrix $A$, and, as above, denote the set of seed vertices by $S=\cup_k S_k$.
Define the set of feasible block assignment functions
\begin{equation*} \begin{aligned}
\mathcal{B}&=\mathcal{B}(\nvec,b,S) \\
&:=\{\phi:V\rightarrow [K] \text{ s.t. for all }k\in[K],\, |\phi^{-1}(k)|=n_k,\text{ and }\phi(i)=b(i)\text{ for all }i\in S\}.
\end{aligned} \end{equation*}
The maximum likelihood estimator of $b\in\mathcal{B}$
is any member of the set of functions
\begin{align}
\label{eq:maxlike}
\bhat &= \arg \max_{\phi\in\mathcal{B}}
        \prod_{\{i,j\}\in\binom{V}{2}}
        \Lambda_{\phi(i),\phi(j)}^{A_{i,j}}
        (1-\Lambda_{\phi(i),\phi(j)})^{1-A_{i,j}} \notag\\
&= \arg \max_{\phi\in\mathcal{B}}
        \sum_{\{i,j\}\in\binom{V}{2}}
        A_{i,j}\log\left( \frac{\Lambda_{\phi(i),\phi(j)}}
                        {1-\Lambda_{\phi(i),\phi(j)}}\right) \notag\\
&= \arg \max_{\phi\in\mathcal{B}}
        \sum_{\{i,j\}\in\binom{U}{2}}
        A_{i,j}\log\left(\frac{\Lambda_{\phi(i),\phi(j)}}
                {1-\Lambda_{\phi(i),\phi(j)}}\right)
        + \sum_{(i,j)\in S\times U}
        A_{i,j}\log\left(\frac{\Lambda_{b(i),\phi(j)}}
                {1-\Lambda_{b(i),\phi(j)}}\right),
\end{align}
where the second equality follows from independence of the
edges and splitting the edges in the sum according to whether or not they
are incident to a seed vertex.
We can reformulate~\eqref{eq:maxlike} as a graph matching problem
by identifying $\phi$ with a permutation matrix $P$:
\begin{definition}
\label{def:GM}
Let $G_1$ and $G_2$ be two $n$-vertex graphs with respective adjacency matrices $A$ and $B$.
The Graph Matching Problem for aligning $G_1$ and $G_2$ is
\begin{equation*}
\min_{P\in\Pi_n}\|AP-PB\|_F,
\end{equation*}
where $\Pi_n$ is defined to be the set of all $n\times n$ permutation matrices.
\end{definition}
Incorporating seed vertices (i.e., vertices whose correspondence across $G_1$ and $G_2$ is known {\it a priori}) into the graph matching problem is immediate \citep{FAP}.
Letting the seed vertices be (without loss of generality) $S=\{1,2,\dots,m\}$ in both graphs, the seeded graph matching (SGM) problem is
\begin{equation} \label{eq:seedgraphmatch}
\min_{P\in\Pi_{\fu}}\left\|A(I_m\oplus P)-(I_m\oplus P)B\right\|_F,
\end{equation}
where
$$I_m\oplus P=\begin{bmatrix}I_m & 0\\ 0& P  \end{bmatrix}.$$
Setting $B \in \R^{n \times n}$ to be the log-odds matrix
\begin{equation}
\label{eq:logodds}
B_{i,j}:=\LL{i}{j},
\end{equation}
observe that the optimization problem in Equation~\eqref{eq:maxlike} is equivalent to that in~\eqref{eq:seedgraphmatch} if we view $B$ as encoding a weighted graph. Hence, we can apply known graph matching algorithms to approximately find $\bhat$.

Decomposing $A$ and $B$ as
\renewcommand{\kbldelim}{[}
\renewcommand{\kbrdelim}{]}
\[
  A = \kbordermatrix{
    & m & \fu  \\
    m & A^{(1,1)} & A^{(1,2)} \\
    \fu & A^{(2,1)} & A^{(2,2)})
  }\hspace {10mm}   B = \kbordermatrix{
    & m & \fu  \\
    m & B^{(1,1)} & B^{(1,2)} \\
    \fu & B^{(2,1)} & B^{(2,2)}
  }
\]
and using the fact that $P \in \Pi_n$ is unitary,
the seeded graph matching problem is equivalent (i.e., has the same minimizer) to
\begin{equation*}
\min_{P\in\Pi_{\fu}}-\tr\left(A^{(2,2)}P(B^{(2,2)})^\top P^\top\right)-\tr\left((A^{(1,2)})^\top B^{(1,2)}P^\top \right)-\tr\left(A^{(2,1)} (B^{(2,1)})^\top P^\top \right).
\end{equation*}
Thus, we can recast \eqref{eq:maxlike} as a
seeded graph matching problem so that finding
\begin{align*} 
\bhat &= \arg \max_{\phi\in\mathcal{B}}\sum_{\{i,j\}\in\binom{U}{2}}A_{i,j}\log\left(\frac{\Lambda_{\phi(i),\phi(j)}}{1-\Lambda_{\phi(i),\phi(j)}}\right)+\sum_{(i,j)\in S\times U}A_{i,j}\log\left(\frac{\Lambda_{b(i),\phi(j)}}{1-\Lambda_{b(i),\phi(j)}}\right)
\end{align*}
is equivalent to finding
\begin{align}
\label{eq:maxlikgm}
\Phat &=\arg \min_{P\in\Pi_{\fu}}-\frac{1}{2}\tr\left(A^{(2,2)}P(B^{(2,2)})^\top P^\top\right)-\tr\left((A^{(1,2)})^\top B^{(1,2)}P^\top \right),
\end{align}
as we shall explain below.

With $B$ defined as in \eqref{eq:logodds}, we define
$$\mathcal{Q}=\left\{Q\in\Pi_{\fu}\text{ s.t. }(I_m\oplus Q)B(I_m\oplus Q)^\top=B\right\}.$$
Define an equivalence relation $\sim$ on $\Pi_{\fu}$ via
$P_1\sim P_2$ iff there exists a $Q\in\mathcal{Q}$ such that $P_1=P_2 Q$; i.e.,
$$(I_m\oplus P_1)B(I_m\oplus P_1)^\top=(I_m\oplus P_2Q)B(I_m\oplus P_2Q)^\top=(I_m\oplus P_2)B(I_m\oplus P_2)^\top.$$
Let $\Phat/\sim$ denote the set of equivalence classes of $\Phat$
under equivalence relation $\sim$.
Solving (\ref{eq:maxlike}) is equivalent to solving (\ref{eq:maxlikgm}) in that there is a one-to-one correspondence between $\hat b$ and $\hat P/\sim$:
for each $\phi\in\hat b$ there is a unique $P\in \hat P/\sim$ (with associated permutation $\sigma$) such that $\phi_{|_U}=b_{|_U}\circ\sigma$;
and for each $P\in \hat P/\sim$ (with the permutation associated with $I_m\oplus P$ given by $\sigma$), it holds that $b\circ\sigma\in\hat b$.

\subsection{The $\Lml$ Vertex Nomination Scheme}
\label{S:VN}

The maximum likelihood (ML) vertex nomination scheme proceeds as follows.
First, the SGM algorithm \citep{FAP,JMLR:v15:lyzinski14a} is used to
approximately find an element of $\Phat$, which we shall denote by $P$.
Let the corresponding element of $\bhat$ be denoted by $\phi$.
For any $i,j\in V$ such that $\phi(i)\neq \phi(j)$,
define $\phi_{i\leftrightarrow j}\in \mathcal{B}$ as
$$\phi_{i\leftrightarrow j}(v)=\begin{cases}
\phi(i)&\text{ if }v=j,\\
\phi(j)&\text{ if }v=i,\\
\phi(v)&\text{ if }v\neq i,j;
\end{cases}$$
i.e., $\phi_{i\leftrightarrow j}$ agrees with $\phi$ except that $i$ and $j$ have their block memberships from $\phi$ switched in $\phi_{i\leftrightarrow j}$.
For $i\in U$ such that $\phi(i)=1$, define
$$\eta(i):=\left(\prod_{\substack{j\in U \text{ s.t.}\\ \phi(j)\neq 1}}\frac{\ell(\phi_{i\leftrightarrow j},G)}{\ell(\phi,G)}   \right)^{\frac{1}{\fu-\fu_1}},$$
where, for each $\psi\in\mathcal{B}$, the likelihood $\ell$ is given by
$$\ell(\psi,G)=\prod_{\{i,j\}\in\binom{U}{2}}\Lambda_{\psi(i),\psi(j)}^{A_{i,j}}(1-\Lambda_{\psi(i),\psi(j)})^{1-A_{i,j}}
\prod_{(i,j)\in S\times U}\Lambda_{b(i),\psi(j)}^{A_{i,j}}(1-\Lambda_{b(i),\psi(j)})^{1-A_{i,j}}.$$
A low/high value of $\eta(i)$ is a measure of our confidence that $i$ is/is not in the block of interest.
For $i\in U$ such that $\phi(i)\neq 1$, define
$$\xi(i):=\left(\prod_{\substack{j\in U \text{ s.t.}\\ \phi(j)= 1}}\frac{\ell(\phi_{i\leftrightarrow j},G)}{\ell(\phi,G)}   \right)^{\frac{1}{\fu_1}}.$$
A low/high value of $\xi(i)$ is a measure of our confidence that $i$ is/is not in the block of interest.
We are now ready to define
the maximum-likelihood based nomination scheme
$\mathcal{L}^{\ML}$:
\begin{align*}
\left(\mathcal{L}^{\ML}\right)^{-1}(1)
        &\in \arg \min\{ \eta(v) : \phi(v)=1 \} \\
\left(\mathcal{L}^{\ML}\right)^{-1}(2)
        &\in \arg \min\left\{ \eta(v) :
                v\in U\setminus \left\{(\mathcal{L}^{\ML})^{-1}(1) \right\},
                \phi(v)=1 \right\}\\
&\vdots\\
\left(\mathcal{L}^{\ML}\right)^{-1}(\fu_1)
        &\in \arg\min \left\{ \eta(v) :
                v\in U\setminus \left\{(\mathcal{L}^{\ML})^{-1}(i)
                                        \right\}_{i=1}^{\fu_1-1},
                        \phi(v)=1 \right\} \\
\left(\mathcal{L}^{\ML}\right)^{-1}(\fu_1+1)
        &\in \arg \max \left\{ \xi(v) : \phi(v)\neq1 \right\}\\
\left(\mathcal{L}^{\ML}\right)^{-1}(\fu_1+2)
        &\in \arg \max \left\{ \xi(v) :
                v\in U\setminus \left\{(\mathcal{L}^{\ML})^{-1}(\fu_1+1)
                                \right\}, \phi(v)\neq1 \right\}\\
&\vdots\\
\left(\mathcal{L}^{\ML}\right)^{-1}(\fu)
        &\in \arg \max \left\{ \xi(v) :
                v\in U\setminus \left\{(\mathcal{L}^{\ML})^{-1}(i)
                        \right\}_{i=\fu_1+1}^{\fu-1}, \phi(v)\neq1 \right\}
\end{align*}
Note that in the event that an argmin (or argmax) above contains more than one element, the order in which these elements is nominated should be taken to be uniformly random.

\begin{remark}
\label{rem:parasunknown}
\emph{
In the event that $\Lambda$ is unknown {\it{a priori}},
we can use the block memberships of the seeds $S$ (assumed to be chosen uniformly at random from $V$) to estimate the edge probability matrix $\Lambda$ as
$$\widehat\Lambda_{k,\ell}=\frac{|\{ \{i,j\}\in E\text{ s.t. }i\in S_k,\,j\in S_\ell\}|}{m_km_\ell} \text{ for }k\neq \ell,$$
and
$$\widehat\Lambda_{k,k}=\frac{|\{ \{i,j\}\in E\text{ s.t. }i\in S_k,\,j\in S_k\}|}{\binom{m_k}{2}}.$$
The plug-in estimate $\Bhat$ of $B$, given by
$$\Bhat_{i,j}:=\log\left(\frac{\widehat \Lambda_{b(i),b(j)}}{1-\widehat \Lambda_{b(i),b(j)}}\right),$$
can then be used in place of $B$ in Eq.~\eqref{eq:maxlikgm}.
If, in addition, $\nvec$ is unknown, we can estimate the block sizes
$n_k$ as
$$\nhat_k=\frac{m_k n}{m},$$
for each $k \in [K]$,
and these estimates can be used to determine the block sizes in $\Bhat$.}
\end{remark}

\subsection{The $\LmlRES$ Vertex Nomination Scheme}
\label{s:schemerestrict}

Graph matching is a computationally difficult problem, and there are no known polynomial time algorithms for solving the general graph matching problem for simple graphs.
Furthermore, if the graphs are allowed to be weighted, directed, and loopy, then graph matching is equivalent to the NP-hard quadratic assignment problem.
While there are numerous efficient, approximate graph matching algorithms \citep[see, for example,][and the references therein]{FAQ,FAP,Zaslavskiy2009,jovo}, these algorithms often lack performance guarantees.

Inspired by the restricted-focus seeded graph matching problem considered in \cite{JMLR:v15:lyzinski14a}, we now define the computationally tractable restricted-focus likelihood maximization vertex nomination scheme $\LmlRES$.
Rather than attempting to quickly approximate a solution to the
full graph matching problem as in~\cite{FAQ,FAP,Zaslavskiy2009,jovo},
this approach simplifies the problem by ignoring the edges between unseeded vertices.  An analogous restriction for matching simple graphs was introduced in \cite{JMLR:v15:lyzinski14a}.
We begin by considering the graph matching problem in Eq.~\eqref{eq:maxlikgm}.
The objective function
$$-\frac{1}{2}\tr\left(A^{(2,2)}P(B^{(2,2)})^\top P^\top\right)-\tr\left((A^{(1,2)})^\top B^{(1,2)}P^\top \right)$$
consists of two terms: $-\frac{1}{2}\tr\left(A^{(2,2)}P(B^{(2,2)})^\top P^\top\right),$ which seeks to align the induced subgraphs of the nonseed vertices; and $-\tr\left((A^{(1,2)})^\top B^{(1,2)}P^\top \right),$
which seeks to align the induced bipartite subgraphs between the seed and nonseed vertices.
While the graph matching objective function, Eq.~\eqref{eq:maxlikgm}, is quadratic in $P$, restricting our focus to the second term in Eq.~\eqref{eq:maxlikgm} yields the following {\it linear assignment problem}
\begin{align}
\label{eq:restmaxlikgm}
\tilde P&=\arg \min_{P\in\Pi_{\fu}}-\tr\left((A^{(1,2)})^\top B^{(1,2)}P^\top \right),
\end{align}
which can be efficiently and exactly solved in $O(\fu^3)$ time with the Hungarian algorithm \citep{hungarian,jonker1987shortest}.
We note that, exactly as was the case of $\Phat$ and $\bhat$, finding $\tilde P$ is equivalent to finding
\begin{align*} 
\tilde b&=\arg \max_{\phi\in\mathcal{B}}\sum_{(i,j)\in S\times U}A_{i,j}\log\left(\frac{\Lambda_{b(i),\phi(j)}}{1-\Lambda_{b(i),\phi(j)}}\right),
\end{align*}
in that there is a one-to-one correspondence between $\tilde b$ and $\tilde P/\sim$.

The $\LmlRES$ scheme proceeds as follows.
First, the linear assignment problem, Eq.~\eqref{eq:restmaxlikgm},
is exactly solved using, for example, the Hungarian algorithm \citep{hungarian} or the path augmenting algorithm of \cite{jonker1987shortest}, yielding $P\in\tilde P$.
Let the corresponding element of $\tilde b$ be denoted by $\phi.$
For $i\in U$ such that $\phi(i)=1$, define
$$\tilde \eta(i):=\left(\prod_{\substack{j\in U \text{ s.t.}\\ \phi(j)\neq 1}}\frac{\ell_R(\phi_{i\leftrightarrow j},G)}{\ell_R(\phi,G)}   \right)^{\frac{1}{\fu-\fu_1}},$$
where, for each $\psi\in\mathcal{B}$, the {\it restricted} likelihood $\ell_R$ is defined via
$$\ell_R(\psi,G)=
\prod_{(i,j)\in S\times U}\Lambda_{b(i),\psi(j)}^{A_{i,j}}(1-\Lambda_{b(i),\psi(j)})^{1-A_{i,j}}.$$
As with $\mathcal{L}^{\ML},$ a low/high value of $\tilde \eta(i)$ is a measure of our confidence that $i$ is/is not in the block of interest.
For $i\in U$ such that $\phi(i)\neq 1$, define
$$\tilde\xi(i):=\left(\prod_{\substack{j\in U \text{ s.t.}\\ \phi(j)= 1}}\frac{\ell_R(\phi_{i\leftrightarrow j},G)}{\ell_R(\phi,G)}   \right)^{\frac{1}{\fu_1}}.$$
As before, a low/high value of $\tilde\xi(i)$ is a measure of our confidence that $i$ is/is not in the block of interest.
We are now ready to define $\LmlRES$:
\begin{align*}
\left( \LmlRES \right)^{-1}(1)
        &\in \arg \min\{ \etatilde(v) : \phi(v)=1 \} \\
\left( \LmlRES \right)^{-1}(2)
        &\in \arg \min\left\{ \etatilde(v) :
                v\in U\setminus \left\{(\LmlRES)^{-1}(1) \right\},
                \phi(v)=1 \right\}\\
&\vdots\\
\left( \LmlRES \right)^{-1}(\fu_1)
        &\in \arg\min \left\{ \etatilde(v) :
                v\in U\setminus \left\{(\LmlRES)^{-1}(i)
                                        \right\}_{i=1}^{\fu_1-1},
                        \phi(v)=1 \right\} \\
\left( \LmlRES \right)^{-1}(\fu_1+1)
        &\in \arg \max \left\{ \xitilde(v) : \phi(v)\neq1 \right\}\\
\left( \LmlRES \right)^{-1}(\fu_1+2)
        &\in \arg \max \left\{ \xitilde(v) :
                v\in U\setminus \left\{(\LmlRES)^{-1}(\fu_1+1)
                                \right\}, \phi(v)\neq1 \right\}\\
&\vdots\\
\left( \LmlRES \right)^{-1}(\fu)
        &\in \arg \max \left\{ \xitilde(v) :
                v\in U\setminus \left\{(\LmlRES)^{-1}(i)
                        \right\}_{i=\fu_1+1}^{\fu-1}, \phi(v)\neq1 \right\}
\end{align*}
Note that, as before, in the event that the argmin (or argmax) in the definition of $\LmlRES$ contains more than one element above, the order in which these elements are nominated should be taken to be uniformly random.

Unlike $\Lml,$ the restricted focus scheme $\LmlRES$
is feasible even for comparatively large graphs
(up to thousands of nodes, in our experience).
However, we will see in Section~\ref{sec:experiments} that the extra information available to $\Lml$---the adjacency structure among the nonseed vertices---leads to superior precision in the $\Lml$ nomination lists
as compared to $\LmlRES$.
We next turn our attention to proving the consistency of the $\Lml$ and
$\LmlRES$ schemes.


\section{Consistency of $\mathcal{L}^{\ML}$ and $\mathcal{L}^{\ML}_R$}
\label{S:consis}

In this section, we state theorems ensuring the consistency of the vertex nomination schemes $\mathcal{L}^{\ML}$ (Theorem \ref{thm:mlconsis}) and $\mathcal{L}^{\ML}_R$ (Theorem \ref{thm:restmlconsis}).
For the sake of expository continuity, proofs are given in the Appendix.
We note here that in these Theorems, the parameters of the underlying block model are assumed to be known {\it a priori}.
In Section \ref{sec:parasunknown}, we prove the consistency of $\Lml$ and $\LmlRES$ in the setting where the model parameters are unknown and must be estimated, as in Remark~\ref{rem:parasunknown}.

Let $G\sim \SBM(K,\nvec,b, \Lambda)$ with associated adjacency matrix $A$, and let $B$ be defined as in \eqref{eq:logodds}.
For each $P\in\Pi_{\fu}$ (with associated permutation $\sigma$) and $k,\ell \in[K]$, define
$$\epsilon_{k,\ell}=\epsilon_{k,\ell}(P)=|\{v\in U_k\text{ s.t. }\sigma(v)\in U_\ell   \}|$$
to be the number of vertices in $U_k$ mapped to $U_\ell$ by $I_m\oplus P$, and for each $k\in[K]$ define
$$\epsilonout{k}(P):=\epsilonout{k} = \sum_{\ell \neq k} \epsilon_{k,\ell}.$$
Before stating and proving the consistency of $\Lml$, we first establish some necessary notation.
Note that in the definitions and theorems presented next, all values implicitly depend on $n$, as $\Lambda=\Lambda_n$ is allowed to vary in $n$.
Let $L$ be the set of distinct entries of $\Lambda$, and define
\begin{align}
\label{eq:abc}
&\alpha = \min_{\{k,\ell\}\text{ s.t. } k\neq \ell}|\Lambda_{k,k} - \Lambda_{k,\ell}|
\hspace{5mm}
\beta = \min_{\{k,\ell\}\text{ s.t. } k\neq \ell }|B_{k,k} - B_{k,\ell}|
\hspace{5mm}c=\max_{i,j,k,\ell}|B_{i,j}-B_{k,\ell}|,\\
\label{eq:gk}
&\hspace{20mm}\gamma = \min_{x,y\in L}|x-y|,\hspace{5mm}\kappa = \min_{x,y\in L}\left|\log\left(\frac{x}{1-x}\right)-\log\left(\frac{y}{1-y}\right)\right|.
\end{align}

\begin{theorem}
\label{thm:mlconsis}
Let $G\sim \SBM(K,\nvec,b,\Lambda)$ and assume that
\begin{itemize}
\item[i.] $K=o(\sqrt{n})$;
\item[ii.] $\Lambda \in [0,1]^{K \times K}$ is such that for all $k,\ell\in[K]$ with $k\neq \ell$, $\Lambda_{k,k}\neq\Lambda_{k,\ell};$
\item[iii.] For each $k\in[K]$, $\fu_k=\omega(\sqrt{n})$, and $m_k=\omega(\log \fu_k)$;
\item[iv.] $\frac{c^2}{\alpha\beta\kappa\gamma}=\Theta(1)$.
\end{itemize}
Then it holds that $\lim_{n\rightarrow\infty}\e \AP(\Lml)=1$, and $\Lml$ is a
consistent nomination scheme.
\end{theorem}

A proof of Theorem~\ref{thm:mlconsis} is given in the Appendix.

\begin{remark}
\emph{
There are numerous assumptions akin to those in Theorem~\ref{thm:mlconsis}
under which we can show that $\Lml$ is consistent.
Essentially, we need to ensure that if we define
$\mathcal{P}' = \{ P \in \Pi_{\fu} : \epsilonout{1}(P)=\Theta(\fu_1)\}$,
then
$ \p\left( \exists\,\, P \in \mathcal{P}'\text{ s.t. } 
X_P \le 0 \right)
        \ $ is summably small, from which it follows that $\epsilonout{1}=o(\fu_1)$ with high probability, which is enough to ensure the desired consistency of~$\mathcal{L}^{\ML}$.}
\end{remark}

Consistency of $\LmlRES$ holds under similar assumptions.
\begin{theorem}
\label{thm:restmlconsis}
Let $G\sim \SBM(K,\nvec,b,\Lambda)$.  Under the following assumptions
\begin{itemize}
\item[i.] $K=\Theta(1)$;
\item[ii.] $\Lambda \in [0,1]^{K \times K}$ is such that for all $k,\ell \in[K]$ with $k\neq \ell$, $\Lambda_{k,k}\neq\Lambda_{k,\ell};$
\item[iii.] For each $k\in[K]$, $\fu_k=\omega(\sqrt{n})$, and $m_k=\omega(\log \fu_k)$;
\item[iv.] $\frac{c^2}{\alpha\beta\kappa\gamma}=\Theta(1)$;
\end{itemize}
it holds that $\lim_{n\rightarrow\infty}\e \AP(\mathcal{L}^{\ML})=1$, and $\mathcal{L}^{\ML}$ is a consistent nomination scheme.
\end{theorem}
A proof of this Theorem can be found in the Appendix.


\section{Consistency of $\Lml$ and $\LmlRES$ When the Model Parameters are Unknown}
\label{sec:parasunknown}

If $\Lambda$ is unknown {\it a priori}, then the seeds can be used to estimate $\Lambda$ as $\widehat \Lambda$, and $n_i$ as $\nhat$ for each $i\in[K]$.
In this section, we will prove analogues of the consistency
Theorems~\ref{thm:mlconsis} and~\ref{thm:restmlconsis}
in the case where $\Lambda$ and $\nvec$ are estimated using seeds.
In Theorems~\ref{thm:hatconsistency} and~\ref{thm:rhatconsistency} below, we prove that under mild model assumptions, both $\Lml$ and $\LmlRES$ are consistent
vertex nomination schemes, even when the seed vertices form a vanishing fraction of the graph.

We now state the consistency result analogous to Theorem~\ref{thm:mlconsis},
this time for the case where we estimate $\Lambda$ and $\nvec$.
The proof can be found in the Appendix.

\begin{theorem}
\label{thm:hatconsistency}
Let $\Lambda\in\mathbb{R}^{K\times K}$ be a fixed, symmetric, block probability matrix
satisfying
\begin{itemize}
\item[i.] $K$ is fixed in $n$;
\item[ii.] $\Lambda \in [0,1]^{K \times K}$ is such that for all $k,\ell \in[K]$ with $k\neq \ell$, $\Lambda_{k,k}\neq\Lambda_{k,\ell};$
\item[iii.] For each $k\in[K],$ $n_k=\Theta(n)$ and $m_k=\omega(n^{2/3}\log(n))$;
\item[iv.] $\alpha$ and $\gamma$ defined as in~\eqref{eq:abc} and~\eqref{eq:gk}
        are fixed in $n$.
\end{itemize}
Suppose that the model parameters of $G\sim(K,\nvec,b,\Lambda)$ are estimated as in Remark \ref{rem:parasunknown} yielding log-odds matrix estimate $\widehat B$ and estimated block sizes $\nhat = (\nhat_1,\nhat_2,\dots,\nhat_K)^T$.
If $\mathcal{L}^{\ML}$ is run on $A$ and $\widehat B$ using the block sizes given by $\nhat$, then under the above assumptions
it holds that $\lim_{n\rightarrow\infty}\e \AP(\mathcal{L}^{\ML})=1$, and $\mathcal{L}^{\ML}$ is a consistent nomination scheme.
\end{theorem}

We now state the analogous consistency result to Theorem~\ref{thm:restmlconsis} when we estimate $\Lambda$ and $\nvec$.
The proof is given in the Appendix.

\begin{theorem}
\label{thm:rhatconsistency}
Let $\Lambda\in\mathbb{R}^{K\times K}$ be a fixed, symmetric, block probability matrix
satisfying
\begin{itemize}
\item[i.] $K$ is fixed in $n$;
\item[ii.] $\Lambda \in [0,1]^{K \times K}$ is such that for all $k,\ell \in[K]$ with $k \neq \ell$, $\Lambda_{k,k}\neq\Lambda_{k,\ell};$
\item[iii.] For each $k\in[K]$ s.t. $k\neq1,$ $n_k=\Theta(n)$ and $m_k=\omega(n^{2/3}\log(n))$;
\item[iv.] $n_1=\Theta(n)$ and $m_1=\omega(n^{4/5})$;
\item[v.] $\alpha$ and $\gamma$ defined at (\ref{eq:abc}) and (\ref{eq:gk}) are fixed in $n$.
\end{itemize}
Suppose that the model parameters of $G\sim(K,\nvec,b,\Lambda)$ are estimated as in Remark \ref{rem:parasunknown} yielding $\widehat B$ and estimated block sizes $\nhat = (\nhat_1,\nhat_2,\dots,\nhat_K)^T$.
If $\mathcal{L}^{\ML}$ is run on $A$ and $\widehat B$ using block sizes given by $\nhat$, then under the above assumptions
it holds that $\lim_{n\rightarrow\infty}\e \AP(\mathcal{L}^{\ML})=1$ and $\mathcal{L}^{\ML}$ is a
consistent nomination scheme.
\end{theorem}

The two preceding theorems imply that vertex nomination is possible even
when the number of seeds is a vanishing fraction of the vertices in the
graph. Indeed, we find that in practice, accurate nomination is possible
even with just a handful of seed vertices.
See the experiments presented in Section~\ref{sec:experiments}.


\section{Model Generalizations}
\label{S:gen}
Network data rarely appears in isolation. In the vast majority of use cases,
the observed graph is richly annotated with information about the vertices
and edges of the network.
For example, in a social network, in addition to information about
which users are friends, we may have vertex-level information in the
form of age, education level, hobbies, etc.
Similarly, in many networks, not all edges are created equal. 
Edge weights may encode the strength of a relation,
such as the volume of trade between two countries.
In this section, we sketch how the $\Lml$ and $\LmlRES$ vertex nomination
schemes can be extended to such annotated networks
by incorporating edge weights and vertex features.
To wit, all of the theorems proven above translate {\it mutatis mutandis}
to the setting in which $G$ is a drawn from a bounded canonical exponential
family stochastic block model.
Consider a single parameter exponential family of distributions whose density can be expressed in canonical form as
$$f(x|\theta)=h(x)e^{T(x)\theta-\mathcal{A}(\theta)}.$$
We will further assume that $h(x)$ has bounded support.
We define
\begin{definition}
\label{def:expfam}
A $\calG_n$-valued random graph $G$ is an instantiation of a $(K,\nvec,b,\Theta)$ bounded, canonical exponential family stochastic block model, written $G\sim \ExpSBM(K,\nvec,b, \Theta)$, if
\begin{itemize}
\item[i.] The vertex set $V$ is partitioned into $K$ blocks, $V_1,V_2,\ldots,V_K$ with sizes $|V_k| = n_k$ for $k=1,2,\dots,K$;
\item[ii.]  The block membership function $b:V\rightarrow[K]$ is such that for each $v\in V$, $v\in V_{b(v)}$;
\item[iii.] The symmetric block parameter matrix $\Theta=[\theta_{k,\ell}]\in\mathbb{R}^{K\times K}$ is such that the $\{i,j\}\in\binom{V}{2}$, $A_{i,j}$ $(=A_{j,i})$ are independent, distributed according to the density
$$f_{A_{i,j}}(x|\theta_{b(i),b(j)})=h(x)e^{T(x)\theta_{b(i),b(j)}-\mathcal{A}(\theta_{b(i),b(j)})}.$$
\end{itemize}
\end{definition}
\noindent Note that the exponential family density is usually written as $h(x)e^{-x\theta-A(\theta)}$,
where $A(\cdot)$ is the log-normalization function.
We have made the notational substitution
to avoid confusion with the adjacency matrix $A$.
If $G\sim \ExpSBM(K,\nvec,b, \Theta)$, analogues to Theorems~\ref{thm:mlconsis},~\ref{thm:restmlconsis},~\ref{thm:hatconsistency} and~\ref{thm:rhatconsistency} follow {\it mutatis mutandis} if we use seeded graph matching to match $\widetilde A=[\widetilde A_{i,j}]:=[T(A_{i,j})]$ to  $B=[B_{i,j}]:=[\theta_{b(i),b(j)}]$;
i.e., under analogous model assumptions, $\Lml$ and $\LmlRES$ are both consistent vertex nomination schemes when the model parameters are known or estimated via seeds.  
The key property being exploited here is that $\e(T(X))$ is a nondecreasing function of $\theta$.  
We expect that results analogous to Theorems~\ref{thm:mlconsis},~\ref{thm:restmlconsis},~\ref{thm:hatconsistency} and~\ref{thm:rhatconsistency} can be shown to hold for more general weight distributions as well, but we do not pursue this further here.

Incorporating vertex features into $\Lml$ and $\LmlRES$ is immediate.
Suppose that each vertex $v\in V$ is accompanied by a $d$-dimensional feature vector $X_v\in\mathcal{R}^d$.
The features could encode additional information about the community structure of the underlying network; for example, if $b(v)=k$ then perhaps $X_v\sim \Norm(\mu_k,\Sigma_k)$ where the parameters of the normal distribution vary across blocks and are constant within blocks.
This setup, in which vertices are ``annotated'' or ``attributed''
with additional information, is quite common.
Indeed, in almost all use cases, some
auxiliary information about the graph is available,
and methods that can leverage this auxiliary information are crucial.
See, for example,~\cite{YanMcALes2013,ZhaLevZhu2015,NewCla2016,FraWol2016}
and citations therein.
We model vertex features as follows.  Conditioning on $b(v)=k$, the feature associated to $v$ is drawn, independently of $A$ and of all other features $X_u$,
from a distribution with density $f_{b(v)}$.
Define the feature matrix $X$ via
\[
  X= \kbordermatrix{
    & d   \\
    m & X^{(m)}  \\
    \fu & X^{(\fu)} 
  },
\]
where $X^{(m)}$ represents the features of the seed vertices in $S$, and $X^{(\fu)}$ the features of the nonseed vertices in $U$.
For each block $k\in[K]$, let $\hat f_k$ be an estimate of the density $f_i$,
and create matrix $F \in \R^{m + \fu}$ given by
\[
  F = \kbordermatrix{
    &    \\
    \fu_1 & \hat f_1( X_1 ) & \hat f_1(X_2) & \cdots & f_1(X_\fu) \\
    \fu_2 & \hat f_2( X_1 ) & \hat f_2(X_2) & \cdots & f_2(X_\fu) \\
     \vdots & \vdots               &               &        & \vdots\\
    \fu_K & \hat f_K( X_1 ) & \hat f_K(X_2) & \cdots & f_K(X_\fu)
  }.
\]
Then we can incorporate the feature density into the seeded graph matching
problem in~\eqref{eq:maxlikgm} by adding a linear factor to the quadratic
assignment problem:
\begin{align}
\label{eq:maxlikgmfeat}
\hat P&=\arg \min_{P\in\Pi_{\fu}}-\frac{1}{2}\tr\left(A^{(2,2)}P(B^{(2,2)})^\top P^\top\right)-\tr\left((A^{(1,2)})^\top B^{(1,2)}P^\top \right)-\lambda \tr FP^\top.
\end{align}
The factor $\lambda\in\mathbb{R}^+$ allows us to weight the features encapsulated in $X$ versus the information encoded into the network topology of $G$.

Vertex nomination proceeds as follows.
First, the SGM algorithm \citep{FAP,JMLR:v15:lyzinski14a} is used to approximately find an element of $\hat P$ in Eq.~(\ref{eq:maxlikgmfeat}), which we shall denote by $P$.
Let the block membership function corresponding to $P$ be denoted $\phi$.
For $i\in U$ such that $\phi(i)=1$, define
$$\eta_F(i):=\left(\prod_{\substack{j\in U \text{ s.t.}\\ \phi(j)\neq 1}}\frac{\ell_F(\phi_{i\leftrightarrow j},G)}{\ell_F(\phi,G)}   \right)^{\frac{1}{\fu-\fu_1}},$$
where, for each $\psi\in\mathcal{B}$, the likelihood $\ell_F$ is given by
\begin{equation*} \begin{aligned}
\ell_F(\psi,G) &=\prod_{\{i,j\}\in\binom{U}{2}}\Lambda_{\psi(i),\psi(j)}^{A_{i,j}}(1-\Lambda_{\psi(i),\psi(j)})^{1-A_{i,j}} \\
&~~~~~~\cdot \prod_{(i,j)\in S\times U}\Lambda_{b(i),\psi(j)}^{A_{i,j}}(1-\Lambda_{b(i),\psi(j)})^{1-A_{i,j}}\prod_{i\in U} \hat{f}_{b(i)}(X_i),
\end{aligned} \end{equation*}
where, for $k\in[K]$, $\hat{f}_k(\cdot)$ is the estimated density
of the $k$-th block features.
Note that here we assume that the feature densities must be estimated,
even when the matrix $\Lambda$ is known.
A low/high value of $\eta_F(i)$ is a measure of our confidence that $i$ is/is not in the block of interest.
For $i\in U$ such that $\phi(i)\neq 1$, define
$$\xi_F(i):=\left(\prod_{\substack{j\in U \text{ s.t.}\\ \phi(j)= 1}}\frac{\ell_F(\phi_{i\leftrightarrow j},G)}{\ell_F(\phi,G)}   \right)^{\frac{1}{\fu_1}}.$$
A low/high value of $\xi_F(i)$ is a measure of our confidence that $i$ is/is not in the block of interest.
The nomination list produced by $\LFEML$ is then realized via:
\begin{align*}
\left( \LFEML \right)^{-1}(1)
        &\in \arg \min\{ \etaF(v) : \phi(v)=1 \} \\
\left( \LFEML \right)^{-1}(2)
        &\in \arg \min\left\{ \etaF(v) :
                v\in U\setminus \left\{(\LFEML)^{-1}(1) \right\},
                \phi(v)=1 \right\}\\
&\vdots\\
\left( \LFEML \right)^{-1}(\fu_1)
        &\in \arg\min \left\{ \etaF(v) :
                v\in U\setminus \left\{(\LFEML)^{-1}(i)
                                        \right\}_{i=1}^{\fu_1-1},
                        \phi(v)=1 \right\} \\
\left( \LFEML \right)^{-1}(\fu_1+1)
        &\in \arg \max \left\{ \xiF(v) : \phi(v)\neq1 \right\}\\
\left( \LFEML \right)^{-1}(\fu_1+2)
        &\in \arg \max \left\{ \xiF(v) :
                v\in U\setminus \left\{(\LFEML)^{-1}(\fu_1+1)
                                \right\}, \phi(v)\neq1 \right\}\\
&\vdots\\
\left( \LFEML \right)^{-1}(\fu)
        &\in \arg \max \left\{ \xiF(v) :
                v\in U\setminus \left\{(\LFEML)^{-1}(i)
                        \right\}_{i=\fu_1+1}^{\fu-1}, \phi(v)\neq1 \right\}
\end{align*}
Note that, once again, in the event that the argmin (or argmax) contains more than one element above, the order in which these elements is nominated should be taken to be uniformly random.

We leave for future work a more thorough investigation of how best to choose
the parameter $\lambda$. We found that choosing
$\lambda$ approximately equal to the number of nonseed vertices yielded
reliably good results, but in general the best choice of $\lambda$ is likely
to be dependent on both the structure of the graph and the available features
(e.g., how well the features actually predict block membership).
We note also that in the case where the feature densities are
not easily estimated
or where we would like to relax our distributional assumptions,
we might consider other terms to use in lieu of $\tr F P^\top$.
For example, let $\muhat_k=\frac{1}{m_k}\sum_{v\in S_k}X_v$ be the empirical
estimate of $\mu_k$, the average feature vector for the seeds in block $k$,
and create let $Y$ be defined via
\[
  Y= \kbordermatrix{
    & d   \\
    \fu_1 & \muhat_1\otimes \vec 1  \\
    \fu_2 & \muhat_2\otimes \vec 1\\
     \vdots&\vdots\\
     \fu_K & \muhat_k\otimes \vec 1
  }.
\]
Incorporating these features into the seeded graph matching problem
similarly to~\eqref{eq:maxlikgmfeat}, we have
\begin{align}
\Phat &=\arg \min_{P\in\Pi_{\fu}}-\frac{1}{2}\tr\left(A^{(2,2)}P(B^{(2,2)})^\top P^\top\right)-\tr\left((A^{(1,2)})^\top B^{(1,2)}P^\top \right)-\lambda\tr(X^{(\fu)} Y^\top P^\top).
\end{align}
We leave further exploration of this and related approaches,
as well as how to deal with categorical data
\citep[e.g., as in][]{NewCla2016}
for future work.


\section{Experiments}
\label{sec:experiments}

To compare the performance of maximum likelihood vertex nomination against
other methods, we performed experiments on five data sets,
one synthetic, the others from linguistics, sociology, political science and ecology.

In all our data sets, we consider vertex nomination both when the
edge probability matrix $\Lambda$ is known and when it must be estimated.
When model parameters are unknown,
$m < n$ seed vertices are selected at random and the edge probability matrix is
estimated based on the subgraph induced by the seeds,
with entries of the edge probability matrix estimated via add-one smoothing.
In the case of synthetic data, the known-parameter case simply
corresponds to the algorithm having access to the parameters used to generate
the data. In this paper, we consider a 3-block stochastic block model (see below),
so the known-parameter case corresponds to the true edge probability matrix being
given. In the case of our real-world data sets, the notion of a ``true''
$\Lambda$ is more hazy. Here, knowing the model parameters
corresponds to using the entire graph, along with the true block memberships,
to estimate $\Lambda$, again using add-one smoothing.
This is, in some sense, the best access we can hope to have to the
model parameters, to the extent that such parameters even exist
in the first place.

\subsection{Simulations}
\label{subsec:sim}

We consider graphs generated from stochastic block models at two different scales.
Following the experiments in~\cite{fishkind2015vertex},
we consider 3-block models, where block sizes are given by
$\nvec = q \cdot (4,3,3)^\top$ for $q = 1,50$, which we term the
small and medium cases, respectively.
In~\cite{fishkind2015vertex},
a third case, with $q=1000$, was also considered, but since ML vertex nomination is
not practical at this scale, we do not include such experiments here,
though we note that $\LmlRES$ can be run successfully on such a graph.
We use an edge probability matrix given by
\begin{equation}
\Lambda(t) =
t \begin{bmatrix} 0.5 & 0.3 & 0.4 \\
                  0.3 & 0.8 & 0.6 \\
                  0.4 & 0.6 & 0.3 \end{bmatrix}
+ (1-t)\begin{bmatrix} 0.5 & 0.5 & 0.5 \\
                  0.5 & 0.5 & 0.5 \\
                  0.5 & 0.5 & 0.5 \end{bmatrix}
\end{equation}
for $t=1,0.3$ respectively in the small and medium cases,
so that the amount of signal present in the graph is smaller
as the number of vertices increases. We consider $m=4,20$ seeds in the
small and medium scales, respectively. For a given choice of
$\nvec,m,t$, we generate a single draw of an SBM with edge probability matrix
$\Lambda(t)$ and block sizes given by $\nvec$. A set of $m$ vertices is
chosen uniformly at random from the first block to be seeds.
Note that this means that the only model parameter that can be
estimated is the intra-block probability for the first block.
For all model parameter estimation in the ML methods (i.e., for the
unknown case of $\Lml$ and $\LmlRES$), we use add-1 smoothing
to prevent inaccurate estimates.
We note that in all conditions, the block of interest (the first block)
is not the densest block of the graph.

Recall that all of the methods under consideration return a list of the
nonseed vertices, which we call a \emph{nomination list},
with the vertices sorted
according to how likely they are to be in the block of interest.
Thus, vertices appearing early in the nomination list are the best candidates
to be vertices of interest.
Figure~\ref{fig:SBM:ranks} compares the performance of canonical,
spectral, maximum likelihood and restricted-focus ML vertex nomination
by looking at (estimates of) their average nomination lists.
The plot shows, for each of the methods under
consideration, an estimate (each based on 200 Monte Carlo replicates)
of the average nomination list.
Each curve describes the empirical
probability that the $k$th-ranked vertex was indeed a vertex of interest.
A perfect method, which on every input correctly places the $n_1$ vertices of
interest in the first $n_1$ entries of the nomination list,
would produce a curve in Figure~\ref{fig:SBM:ranks} resembling a step
function, with a step from 1 to 0 at the $(n_1 + 1)$th rank.
Conversely, a method operating purely at random would yield
an average nomination list that is constant $n_1/n$.
Canonical vertex nomination is shown in gold, ML in blue, restricted-focus ML
in red, and spectral vertex nomination is shown in purple and green.
These two colors correspond, respectively, to spectral VN in which
vertex embeddings are projected to the unit sphere prior to
nomination and in which the embeddings are used as-is.
In sparse networks, the adjacency spectral embedding places all vertices
near to the origin. In such settings, projection to the sphere often
makes cluster structure in the embeddings more easily recoverable.
Dark colors correspond to the known-parameter case, and light colors
correspond to unknown parameters. Note that spectral VN does not make
such a distinction.

Examining the plots,
we see that in the small case, maximum likelihood nomination is
quite competitive with the canonical method, and restricted-focus ML
is not much worse. Somewhat surprising is that these methods perform well
seemingly irrespective of whether or not the model parameters are known,
though this phenomenon is accounted for by the fact that the smoothed
estimates are automatically close to the truth, since $\Lambda$ is
approximately equal to the matrix with all entries $1/2$.
Meanwhile, the small number of nodes is such that
there is little signal available to spectral vertex nomination.
We see that spectral vertex nomination performs approximately at-chance
regardless of whether or not we project the spectral embeddings to the sphere.
10 nodes are not enough to reveal eigenvalue structure
that spectral methods attempt to recover.
In the medium case, where there are 500 vertices, enough signal is present
that reasonable performance is obtained by spectral vertex nomination,
with performance with (purple) and without (green) projection to the sphere
again indistinguishable. The comparative density of the SBM in question
ensures that projection to the sphere is not necessary, and that
doing so does no appreciable harm to nomination.
However, in the medium case, ML-based vertex nomination still appears to
best spectral methods, with the known and unknown cases being nearly
indistinguishable.
We note that in both the small and medium cases all of the methods appear to intersect at an empirical probability of $0.4$.
These intersection points correspond to the transition from the block of interest to the non-interesting vertices: these vertices, about which we are least confident, tend to be nominated correctly at or near chance, which is 40\% in both the small and large cases.

\begin{figure}[t!]
  \centering
  \subfloat[Small scale simulation results]{ \includegraphics[width=0.50\columnwidth]{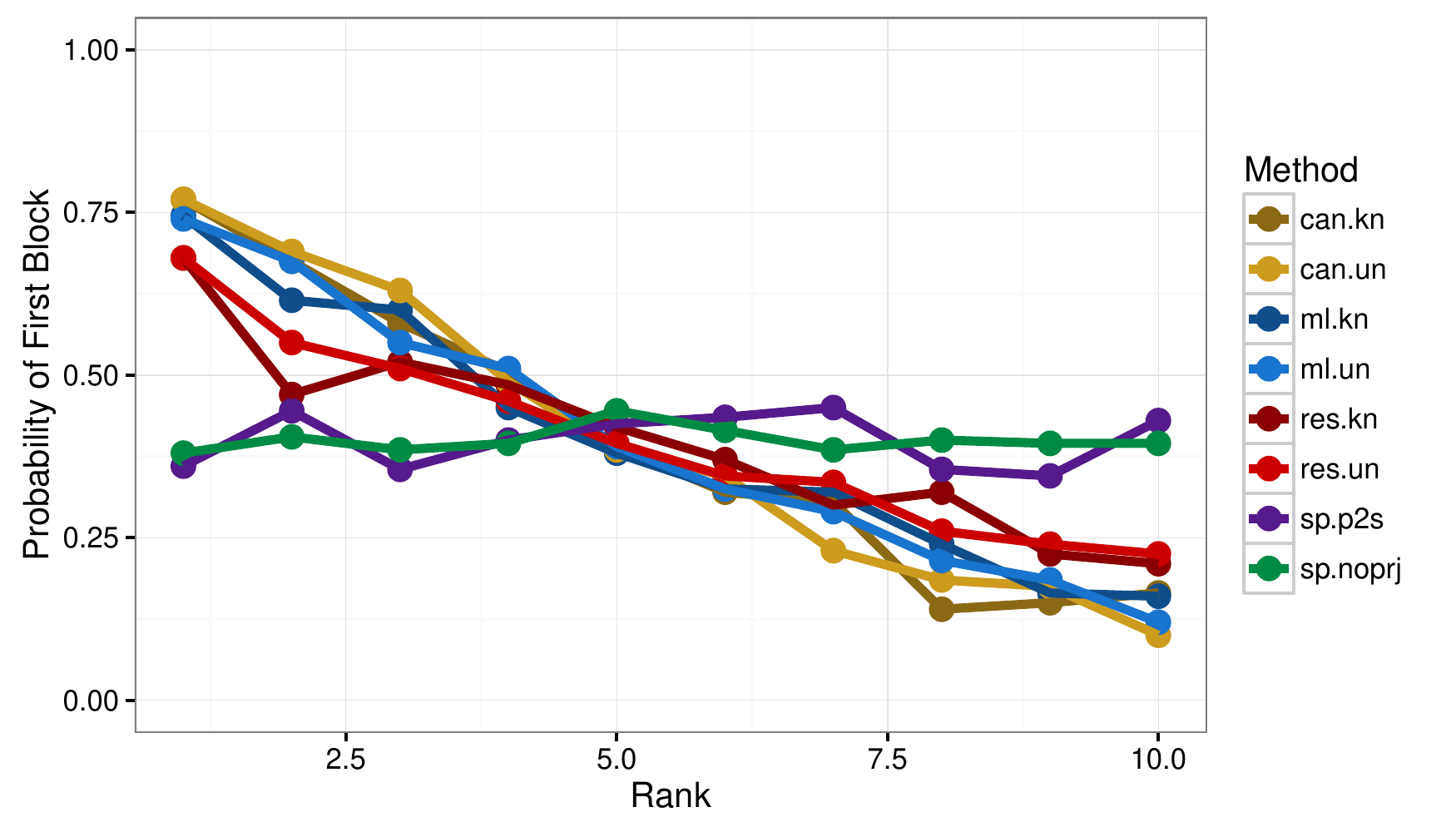} }
  \subfloat[Medium scale simulation results]{ \includegraphics[width=0.50\columnwidth]{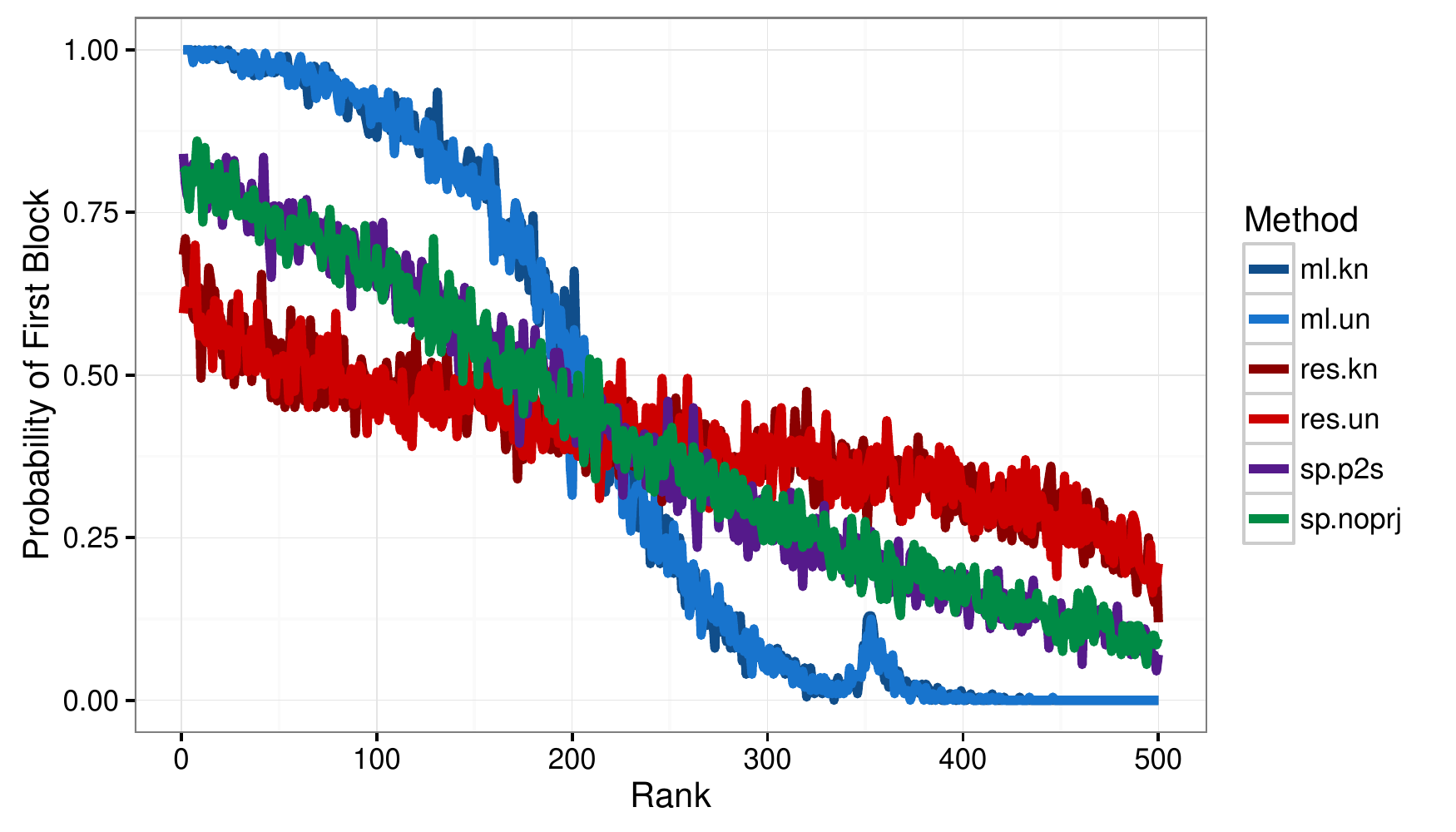} }
  \vspace{-0.2cm}
  \caption{ The mean nomination lists for the
        (a) small and (b) medium stochastic block model experiments
        for the different vertex nomination techniques
        in both the known (dark colors) and unknown (light colors).
        Plot (a) shows performance for the canonical (gold),
        maximum likelihood (blue), restricted-focus maximum likelihood (red)
        and spectral (green and purple) methods.
        Spectral VN both with and without projection to the sphere
        is shown in purple and green, respectively.
        Plot (b) does not include canonical vertex nomination
        due to runtime constraints. }
  \label{fig:SBM:ranks}
  \vspace{-0.25cm}
\end{figure}

A more quantitative assessment of the vertex nomination methods is
contained in Tables~\ref{tab:SBM:aveprec} and~\ref{tab:SBM:ari},
which compare the performance of the methods as assessed by,
respectively, average precision (AP) and adjusted Rand index (ARI).
As defined in Equation~\eqref{eq:def:AP},
AP is a value between 0 and 1, where a value of 1 indicates perfect performance.
ARI~\citep{HubAra1985} measures how well a given partition of a set recovers
some ground truth partition. Here a value of 1 indicates perfect
recovery, while randomly partitioning a data set yields ARI approximately 0
(note that negative ARI is possible).
We include ARI as an evaluation to highlight the fact that
spectral and maximum likelihood vertex nomination do not merely classify
vertices as interesting or not.
Rather, they return a partition of the vertices into clusters.
Canonical vertex nomination, on the other hand, makes no attempt to recover
the full cluster structure of the graph, instead only attempting to classify
vertices according to whether or not they are of interest.
As such, we do not include ARI numbers for canonical vertex nomination.
Turning first to performance in the small graph condition
in Table~\ref{tab:SBM:aveprec},
we see that $\Lcan$ is the best method, so long as the graph in question
is small enough that the canonical method is tractable,
but $\Lml$, regardless of whether or not model parameters are known,
nearly matches canonical VN, and, unlike its canonical counterpart,
scales to graphs with more than a few nodes.
The numbers for $\Lspec$ bear out our observation above,
that the small graphs contain too little information for spectral VN
to act upon, and $\Lspec$ performs approximately at chance, as a result.
It is worth noting that while $\LmlRES$ does not match the
performance of $\Lml$, presumably owing to the fact that the restricted-focus
algorithm does not use all of the information present in the graph,
it still outperforms spectral nomination, and lags $\Lml$ by less than 0.1 AP.

Turning our attention to the medium case,
we see again that $\Lml$ and $\LmlRES$ remain largely impervious to
whether model parameters are known or not, presumably a consequence of
the use of smoothing---we'll see in the sequel that estimation
can be the difference between near-perfect performance and near-chance.
With more vertices, we see that spectral improves above chance,
leaving restricted ML slightly worse,
but spectral still fails to match the performance of ML VN,
even when model parameters are unknown.

In sum, these results suggest that different size graphs
(and different modeling assumptions) call for different vertex nomination
methods. In small graphs, regardless of whether or not model parameters are
known, canonical vertex nomination is both tractable and quite effective.
In medium graphs, maximum likelihood vertex nomination remains tractable
and achieves impressively good nomination.
Of course, for graphs with thousands of vertices, $\Lml$ becomes
computationally expensive, leaving only $\Lspec$ and $\LmlRES$ as options.
We have observed that $\LmlRES$ tends to
lag $\Lspec$ in such large graphs, though increasing the number of seeds
(and hence the amount of information available to $\LmlRES$)
closes this gap considerably.
We leave for future work a more thorough exploration of under what
circumstances we might expect $\LmlRES$ to be competitive with
$\Lspec$ in graphs on thousands of vertices.

\begin{table}[t!]
  \centering
  \begin{tabular}{ r | c | c | c | c | c | c | c | c | }
      \cline{2-9}
      & \multicolumn{4}{ |c| }{ Known }   & \multicolumn{4}{ |c| }{ Unknown } \\
      \cline{2-9}
                & ML    & RES   & SP    & CAN   & ML    & RES & SP    & CAN \\
      \hline
        small & 0.670 & 0.588 & 0.388 & 0.700 & 0.680 & 0.606 & 0.415 & 0.710 \\
        medium & 0.954 & 0.545 & 0.738 & --   & 0.954 & 0.537 & 0.735 & -- \\
        \hline
  \end{tabular}
  \caption{ Empirical estimates of
        mean average precision on the two stochastic block model data
        sets for the four methods under consideration.
        Each data point is the mean of 200 independent trials. }
  \label{tab:SBM:aveprec}
\end{table}

\begin{table}[t!]
  \centering
  \begin{tabular}{ r | c | c | c | c | c | c | c | c | }
      \cline{2-9}
      & \multicolumn{4}{ |c| }{ Known } & \multicolumn{4}{ |c| }{ Unknown } \\
      \cline{2-9}
                & ML    & RES   & SP    & CAN & ML    & RES   & SP    & CAN \\
      \hline
         small & 0.338 & 0.259 & 0.011 & --  & 0.338 & 0.259 & 0.011 & --  \\
         medium & 0.572 & 0.039 & 0.268 & --  & 0.572 & 0.037 & 0.271 & -- \\
        \hline
  \end{tabular}

  \caption{ ARI on the different sized data
        sets for the ML, restricted ML, and spectral methods.
        Each data point is the mean of 200 independent trials.
        Performance of canonical vertex nomination is knot included,
        since canonical vertex nomination makes no attempt to recover all
        three blocks, and thus ARI is not a sensible measure. }
  \label{tab:SBM:ari}
\end{table}

\subsection{Word Co-occurrences}

We consider a linguistic data set consisting of co-occurrences of
54 nouns and 58 adjectives in Charles Dickens' novel
\emph{David Copperfield}~\citep{Newman2006}.
We construct a graph in which
each node corresponds to a word, and an edge connects two nodes
if the two corresponding words occurred adjacent to one another in the text.
The adjacency matrix of this graph is shown in Figure~\ref{fig:adjnoun:A}.
Visual inspection reveals a clear block structure, and that
this block structure is clearly not assortative (i.e., inter-block
edges are more frequent than intra-block edges).
This runs contrary to many commonly-studied data sets and model assumptions.
Figure~\ref{fig:adjnoun:apari} shows the performance of spectral and
maximum-likelihood vertex nomination, measured by (a) average precision
and adjusted Rand index (ARI) at various numbers of seeds.
Each data point is the average over 1000 trials.
In each trial, a set of $m$ seeds was chosen uniformly at random from
the 112 nodes, with the restriction that at least one noun and one adjective
be included in the seed set.
Performance was then measured as the mean average precision
in identifying the adjective block.

\begin{figure}[t!]
  \centering
  \vspace{-5mm}
    \includegraphics[width=0.5\columnwidth]{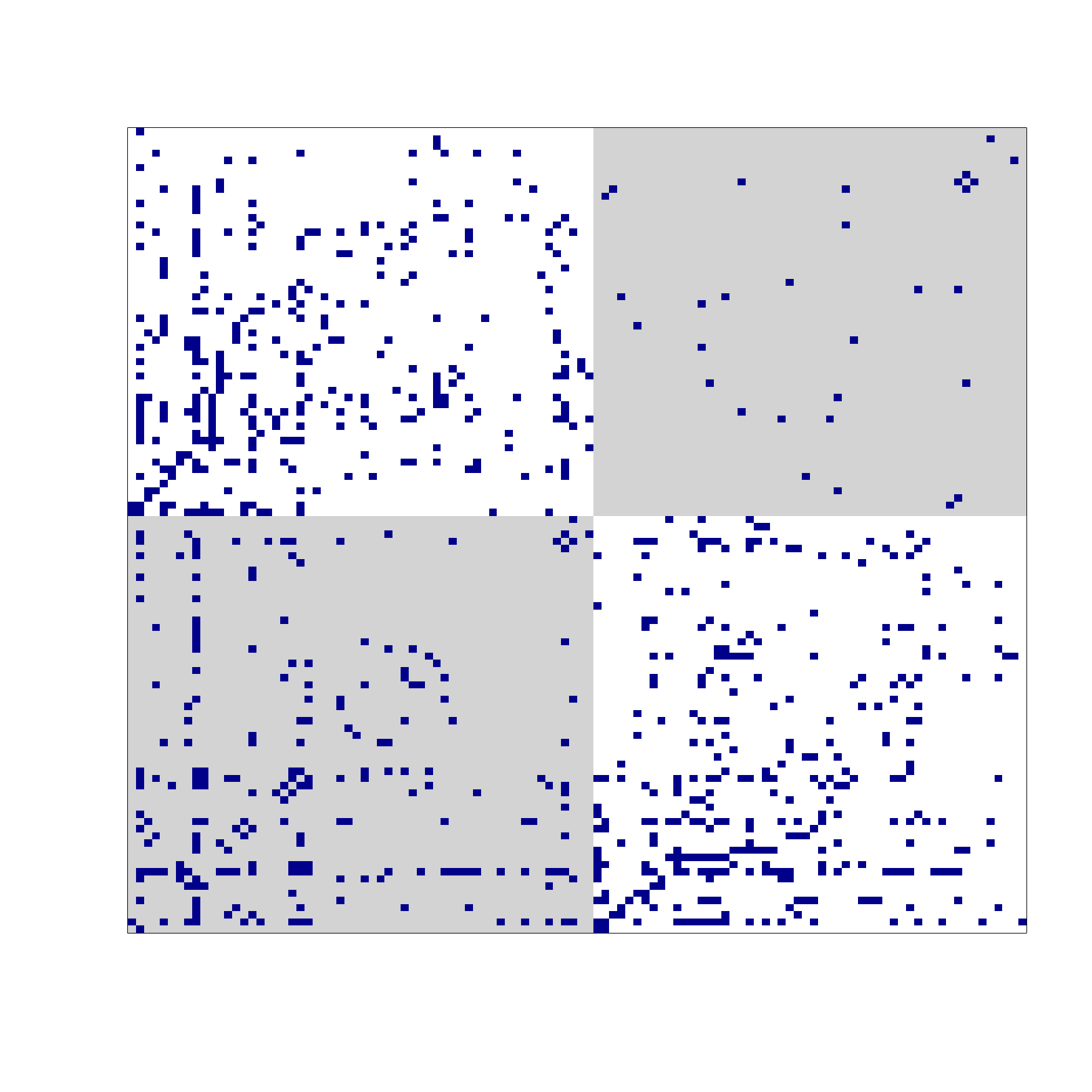}
  \vspace{-5mm}
  \caption{ Adjacency matrix of the linguistic data set,
        arranged to highlight the graph's structure.
        The grey shading indicates the two blocks, with adjectives in
        the lower left and nouns in the upper right.
        Note the disassortative block structure. }
  \label{fig:adjnoun:A}
\end{figure}

Figure~\ref{fig:adjnoun:apari} shows the performance of the VN schemes
under consideration, as a function of the number of seed vertices,
using both known (dark colors) and estimated (light colors) model parameters.
Looking first at AP in Figure~\ref{fig:adjnoun:apari} (a),
we see that ML in the known-parameter case (dark blue) does consistently well,
even with only a handful of seeds, and attains near-perfect performance for
$m \ge 20$. When model parameters must be estimated (light blue), ML is less
dominant, thought it still performs nearly perfectly for $m \ge 20$.
We note the dip in unknown-parameters ML as $m$ increases from $2$ to $5$ to $10$,
a phenomenon we attribute to the bias-variance tradeoff.
Namely, with more seeds available, variance in the estimated model
parameters increases, but for $m < 20$, this increase in variance
is not offset by an appreciable improvement in estimation,
possibly attributable to our use of add-one smoothing.
Somewhat surprisingly, restricted-focus ML performs quite well,
consistently improving on spectral VN in the known parameter
case for $m>2$, and in the unknown parameter case once $m > 10$.
Finally, we turn our attention to spectral VN,
shown in green for the variant in which we project embeddings
to the sphere and in purple for the variant in which we do not.
In contrast to our simulations, the sparsity of this network
makes projection to the sphere a critical requirement for
successful retrieval of the first block.
Without projection to the sphere, spectral VN fails
to rise appreciably above chance performance.

\begin{figure}[t!]
  \centering
  \subfloat[]{ \includegraphics[width=0.50\columnwidth]{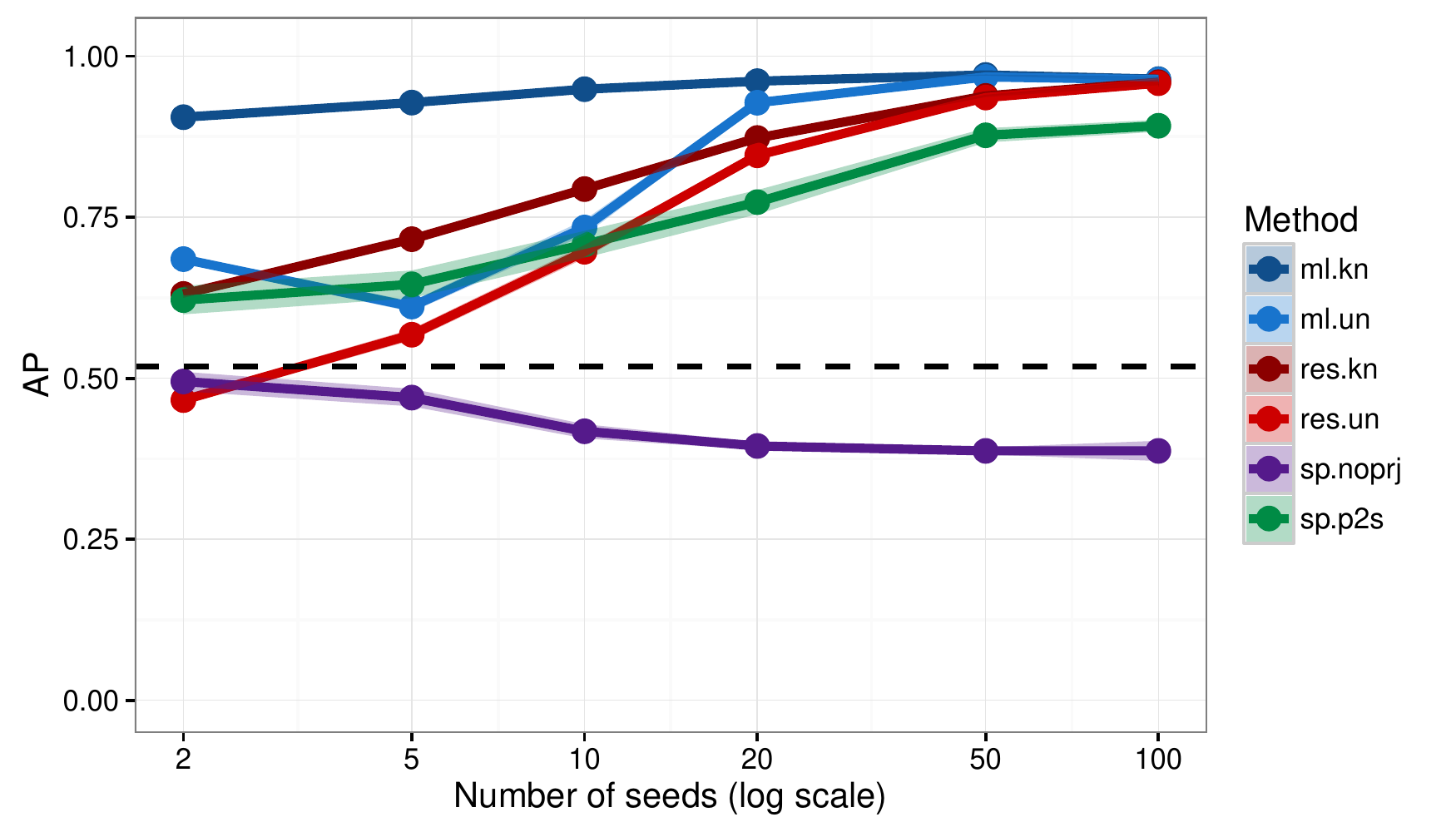}}
  \subfloat[]{ \includegraphics[width=0.50\columnwidth]{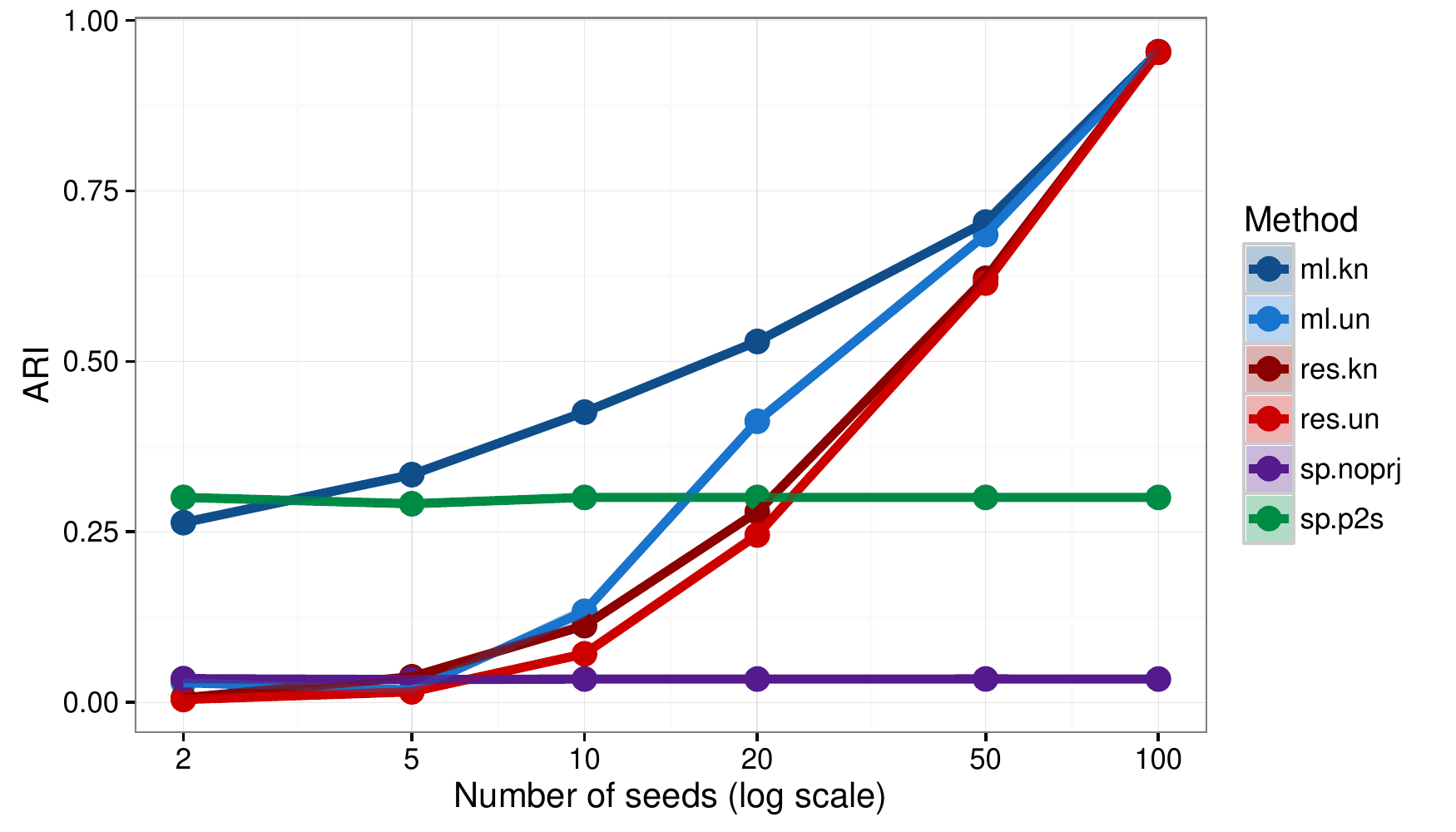} }
  \vspace{-0.2cm}
  \caption{ Performance on the linguistic data set
        as measured by (a) AP and (b) ARI
        as a function of the number of seeds for the
        ML vertex nomination (blue),
        restricted-focus ML (red),
        and spectral vertex nomination with (green)
        and without projection to the sphere (violet),
        when model parameters are
        known (light colors) and unknown (dark colors).
        Each data point is the mean of 1000 Monte Carlo trials,
        and shaded regions indicate two standard deviations of the mean. }
  \label{fig:adjnoun:apari}
\end{figure}

\subsection{Zachary's Karate Club}
We consider the classic sociological data set, Zachary's karate club
network~\citep{Zachary1977}.
The graph, visualized in Figure~\ref{fig:karate:graph},
consists of 34 nodes, each corresponding to a member of a college karate club,
with edges joining pairs of club members according to whether or not
those members were observed to interact consistently outside of the club.
Over the course of Zachary's observation of the group, a conflict emerged
that led to the formation of two factions, led by the individuals numbered 1
and 34 in Figure~\ref{fig:karate:graph},
and these two factions constitute the two blocks in this experiment.
Zachary's karate data set is particularly well-suited for spectral methods.
Indeed, the flow-based model originally proposed by Zachary recovers
factions nearly perfectly,
and visual inspection of the graph (Figure~\ref{fig:karate:graph})
suggests a natural cut separating the two factions.
As such, we expect ML-based vertex nomination to lose out
against the spectral-based method.
Figure~\ref{fig:karate:apari} shows performance
of the two algorithms as measured by ARI and average precision.
We see, as expected,
that spectral performance performs nearly perfectly, irrespective of
the number of seeds.
Surprisingly, maximum likelihood nomination is largely competitive with
spectral VN, but only provided that the model parameters are already
known.
Interesting to note that here again we see the phenomenon
discussed previously in which ML performance with an unknown
edge probability matrix degrades when going from $s=2$ seeds to
$s=5$ before improving again,
with AP comparable to the known case for $s \ge 20$.

\begin{figure}[t!]
  \centering
    \includegraphics[width=0.4\columnwidth]{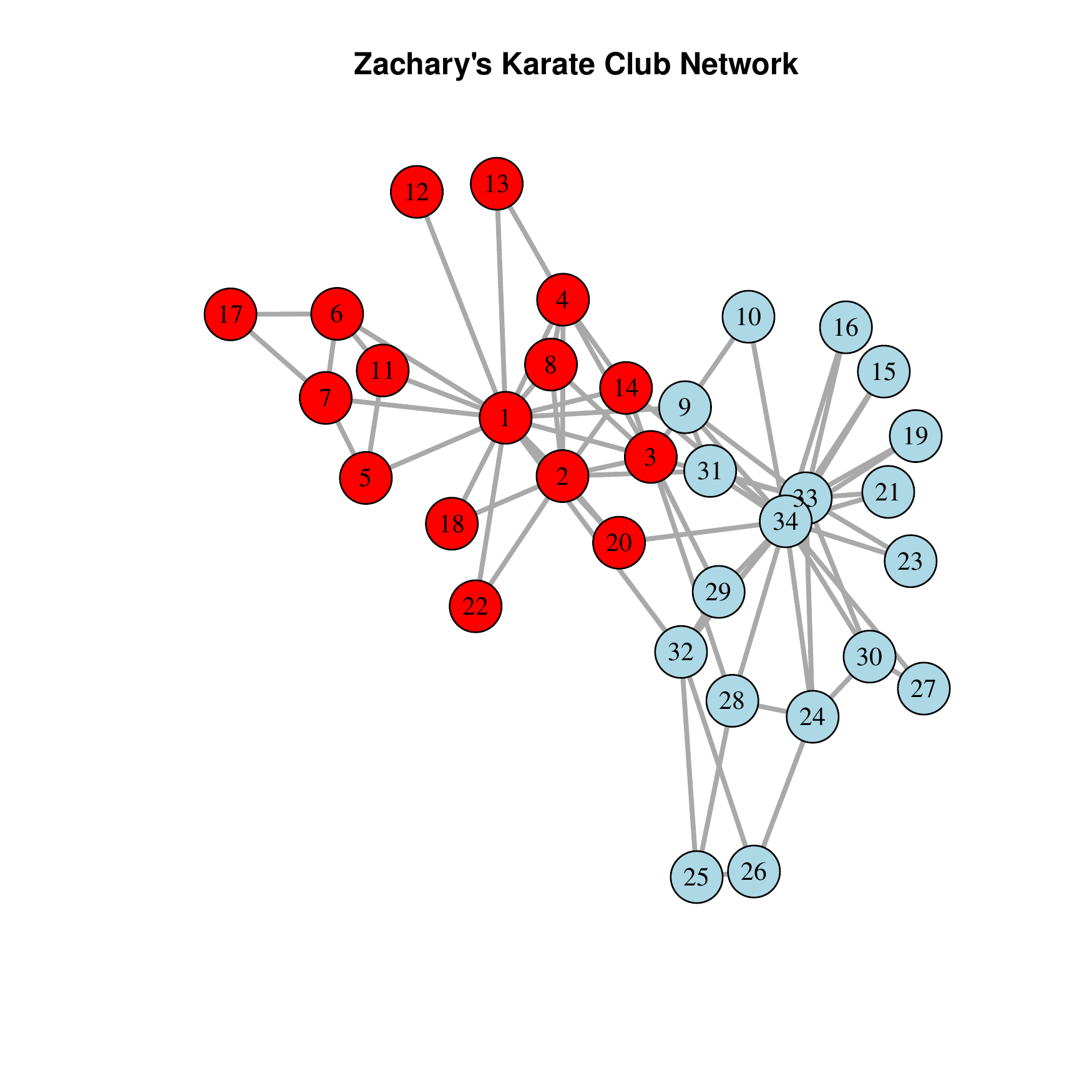}
    \vspace{-10mm}
  \caption{ Visualization of the graph corresponding to Zachary's karate
        club data set. The vertices are colored according to which of the
        two clubs each member chose to join after the schism.
        Our block of interest is in red. }
  \label{fig:karate:graph}
\end{figure}

\begin{figure}[t!]
  \centering
  \subfloat[]{ \includegraphics[width=0.5\columnwidth]{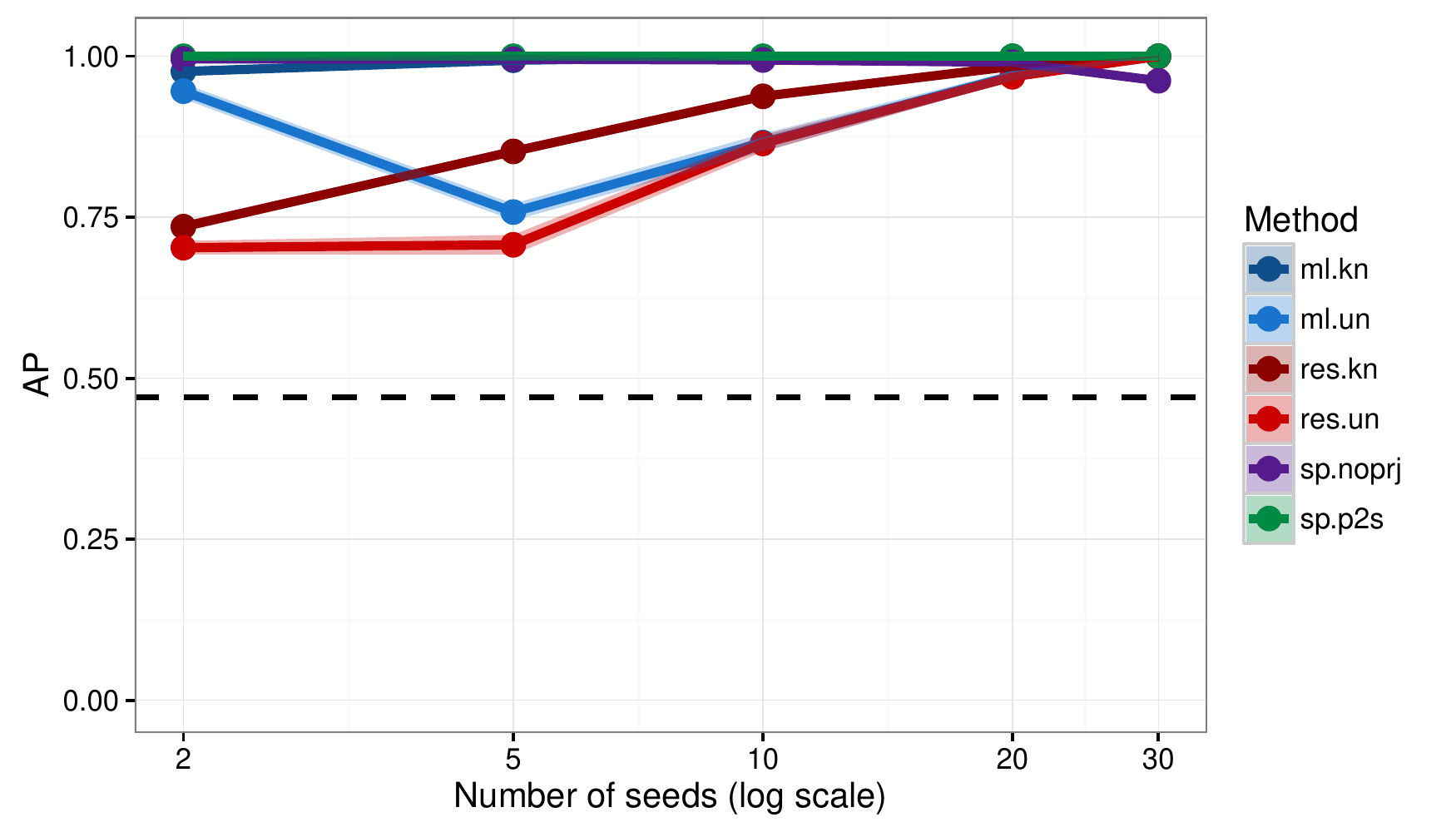} }
  \subfloat[]{ \includegraphics[width=0.5\columnwidth]{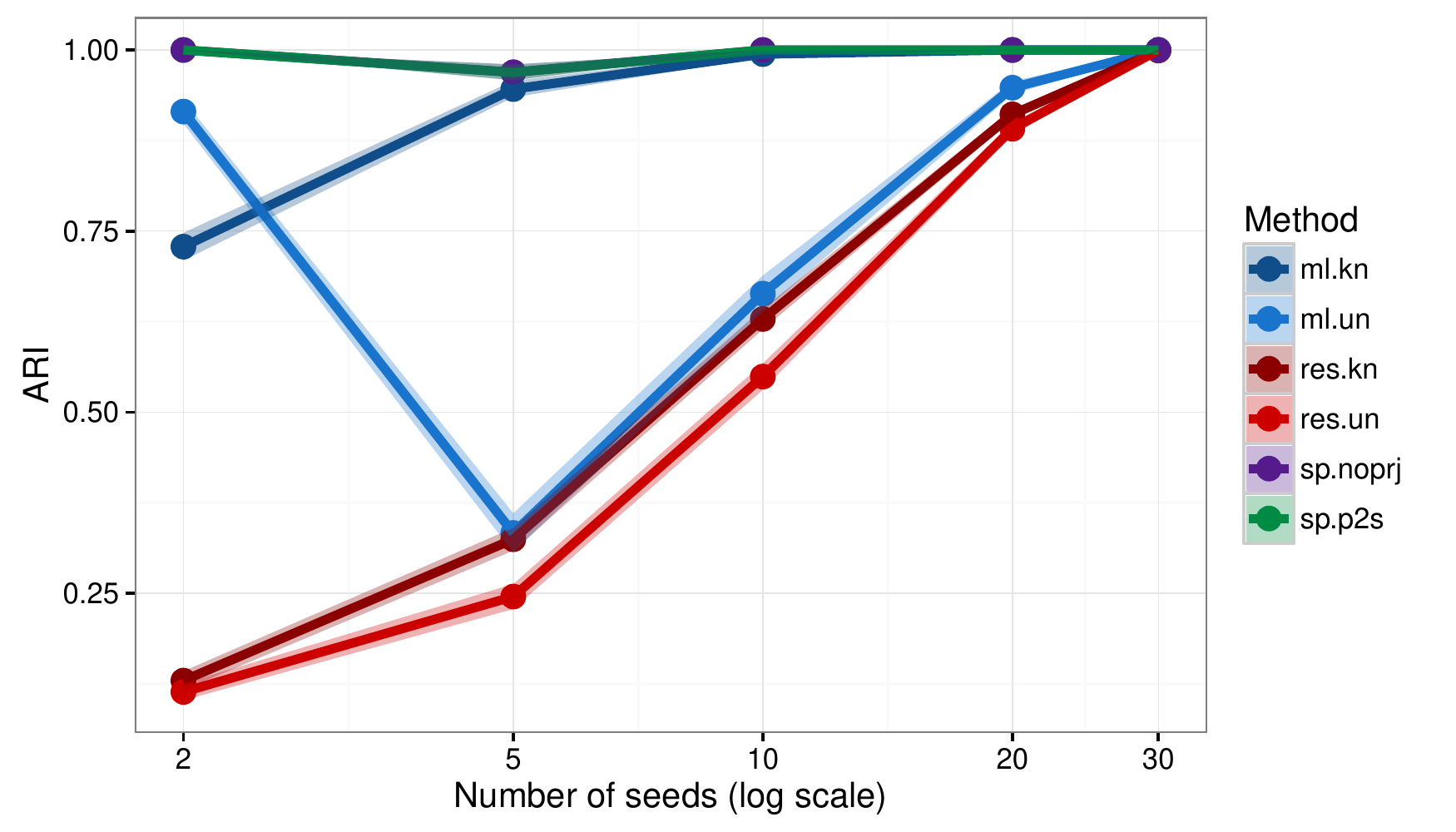} }
  \vspace{-0.2cm}
  \caption{
        Performance on the karate
        data set as a function of the number of seeds for the
        ML vertex nomination (blue),
        restricted-focus ML nomination (red),
        and spectral vertex nomination with (green)
        and without projection to the sphere (violet),
        when model parameters are known (light colors)
        and unknown (dark colors), as measured by (a) AP and (b) ARI.
        The black dashed line indicates chance performance.
        Each observation is the mean of 1000 independent trials,
        with the shaded bars indicating two standard errors of the mean
        in either direction. }
  \label{fig:karate:apari}
\end{figure}

\subsection{Political Blogs}
We consider a network of American political blogs in the lead-up
to the 2004 election~\citep{AdaGla2005},
where an edge joins two blogs if one links to the other,
with blogs classified according to political leaning (liberal vs conservative).
From an initial 1490 vertices, we removed all isolated vertices
to obtain a network of 1224 vertices and 16718 edges.
Figure~\ref{fig:polblog:apari} shows the
performance of the spectral- and ML-based methods in recovering
the liberal block.
We observe first and foremost that the sparsity of this network results in
exceptionally poor performance in both AP and ARI for spectral VN
unless the embeddings are projected to the sphere,
but that spectral vertex nomination is otherwise quite effective
at recovering the liberal block, with
performance nearly perfect for $m > 10$.
Unsurprisingly, ML and its restricted counterpart both perform
approximately at-chance when $m < 10$.
We see that in both the known and unknown cases,
ML VN is competitive with spectral VN for suitably large $m$
($m \ge 50$ for known, $m \ge 500$ for unknown).
As expected in such a sparse network, restricted-focus ML lags
ML VN in the known-parameter case, but surprisingly,
in the unknown-parameter case, restricted ML achieves remarkably
better AP than does ML, a fact we are unable to account for,
though it is worth noting that looking at ARI in
Figure~\ref{fig:polblog:apari} (b), no such gap appears
between ML and its restricted-focus counterpart
in the unknown-parameter case.

\begin{figure}[t!]
  \centering
  \subfloat[]{ \includegraphics[width=0.50\columnwidth]{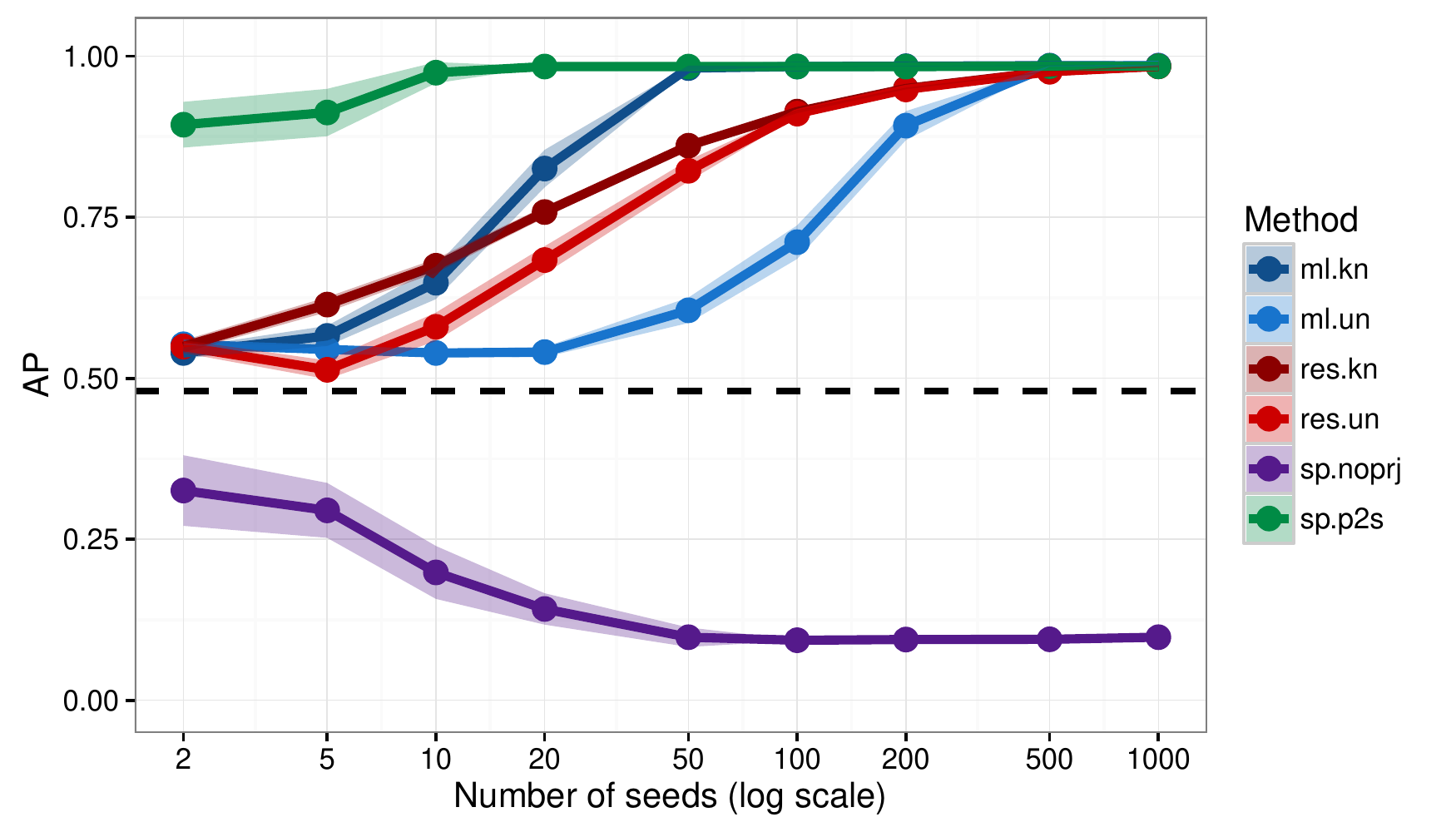} }
  \subfloat[]{ \includegraphics[width=0.50\columnwidth]{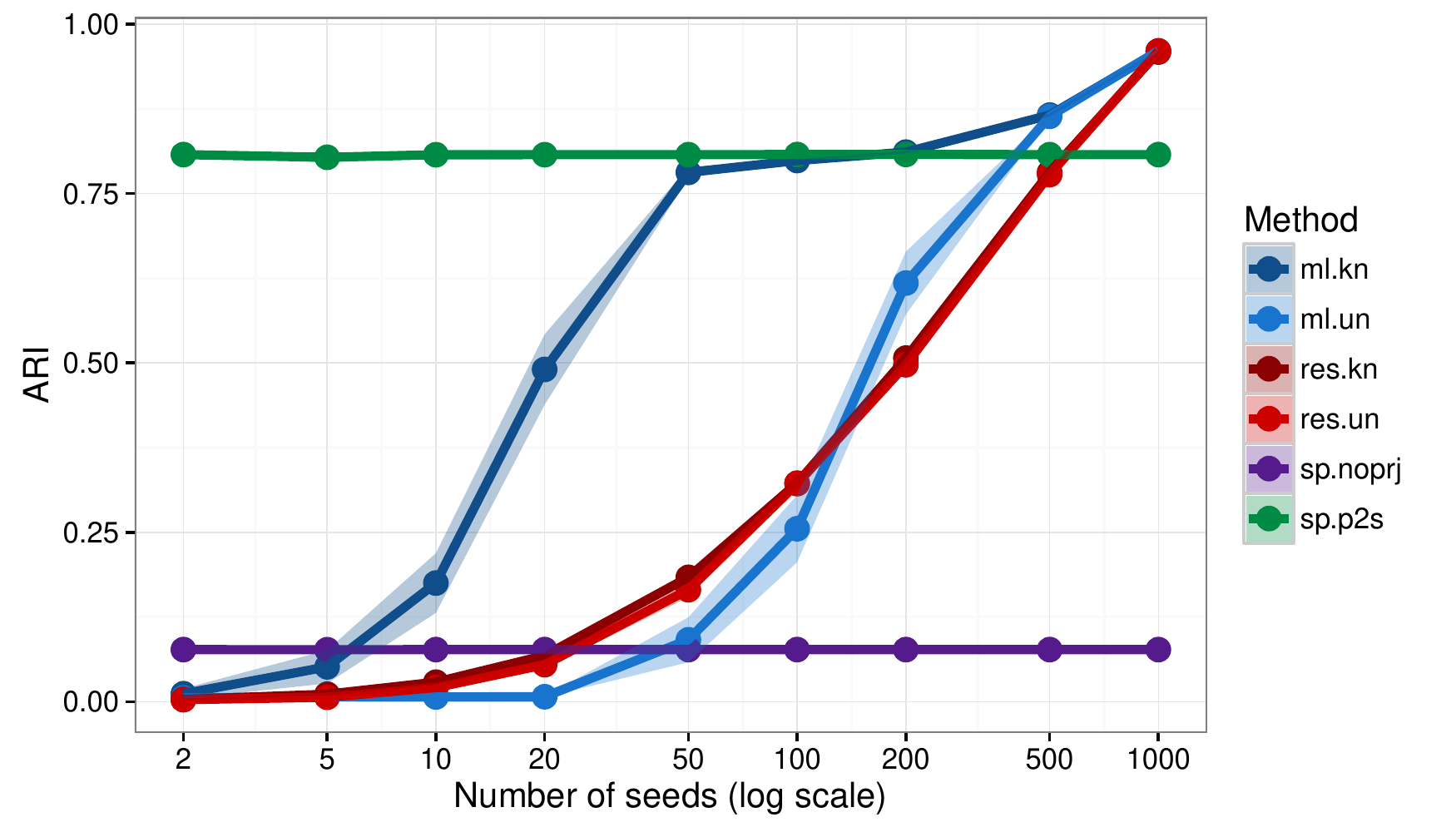} }
  \vspace{-0.2cm}
  \caption{ Performance on the political blogs data set
        as a function of the number of seeds for the
        ML vertex nomination (blue),
        restricted-focus ML (red),
        and spectral vertex nomination with (green)
        and without projection to the sphere (violet),
        when model parameters are
        known (light colors) and unknown (dark colors),
        as measured by (a) AP and (b) ARI. }
  \label{fig:polblog:apari}
\end{figure}

\subsection{Ecological Network}
We consider a trophic network, consisting of 125 nodes and 1907 edges,
in which nodes correspond to (groups of)
organisms in the Florida Bay ecosystem~\citep{UlaBonEgn1997,DeNMrvBat2011},
and an edge joins a pair of organisms if one feeds on the other.
Our features are the (log) mass of organisms.
We take our community of interest to be the 16 different
types of birds in the ecosystem.
This choice makes for an interesting task for several reasons.
Firstly, unlike the other data sets we consider,
our community of interest is a comparatively small fraction of the network---it consists of a mere 16 nodes of 125 in total.
Further, our block of interest is comparatively heterogeneous
in the sense that the roles of the different types of birds
in the Florida Bay ecosystem is quite diverse.
For example, the block of interest includes both raptors and shorebirds,
which feed on quite different collections of organisms.
Finally, it stands to reason that the mass of the organisms in question
might be a crucial piece of information for disambiguating, say, a raptor
from a shark. Thus, we expect that using node features will be crucial
for retrieving the block of interest.

The topology of the Florida Bay network is shown in
Figure~\ref{fig:florida} (a). Note that the block of interest,
indicated in red, has a strongly disassortative structure.
Indeed, all intra-block edges in the red block are incident to the node
corresponding to raptors.
Figure~\ref{fig:florida} (b) summarizes vertex nomination performance
for several methods.
The plot shows performance, as measured by mean average precision (AP),
as a function of the number of seeds for several different
nomination schemes. As in earlier plots, dark colors correspond to model
parameters being known, while light colors correspond to model parameters
being estimated using the seed vertices.
We see immediately that spectral nomination (green and purple) and ML VN (blue)
fail to improve appreciably upon chance performance except when the vast
majority of the vertices' labels are observed.
Like in the linguistic data set presented above,
the disassortative structure of the data appears to cause problems for
spectral nomination.
The failure of ML suggests that no useful information is encoded
in the graph itself, but turning our attention to the curves corresponding
to $\LFEML$ (red) and using only features (gold),
we see that this is not the case.
Indeed, we see that while using features alone achieves a marked improvement
over both spectral and ML-based nomination, using both features and graph
matching in the form of $\LFEML$ yields an additional improvement of
some 0.1 AP in the range of $m=8,16,32$.
This result suggests that there may be cases
where the only reliable way to retrieve vertices of interest is to leverage
both features and graph topology jointly.
\begin{figure}[t!]
  \centering
  \subfloat[]{ \includegraphics[width=0.4\columnwidth]{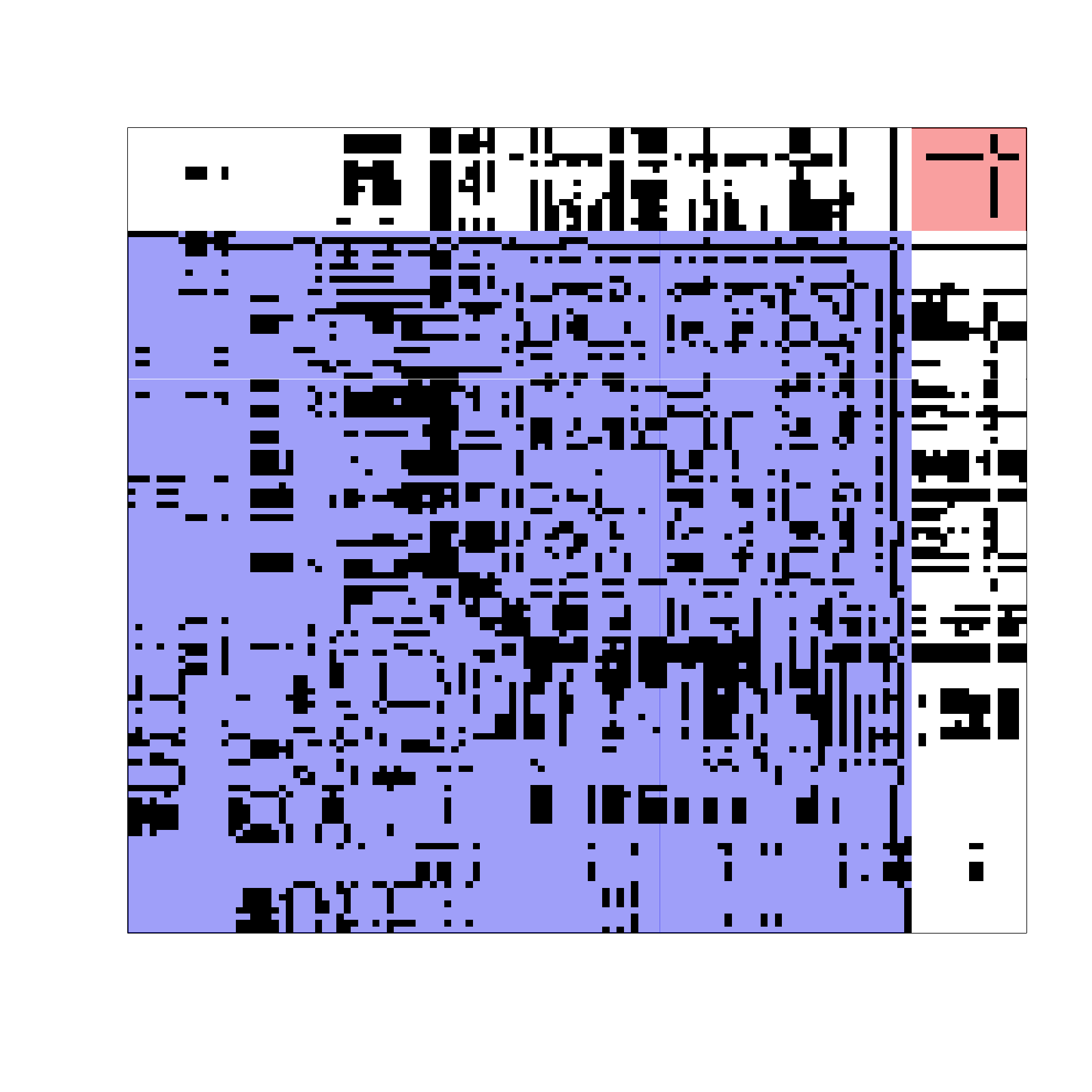} }
  \subfloat[]{ \vspace{-1cm} \includegraphics[width=0.6\columnwidth]{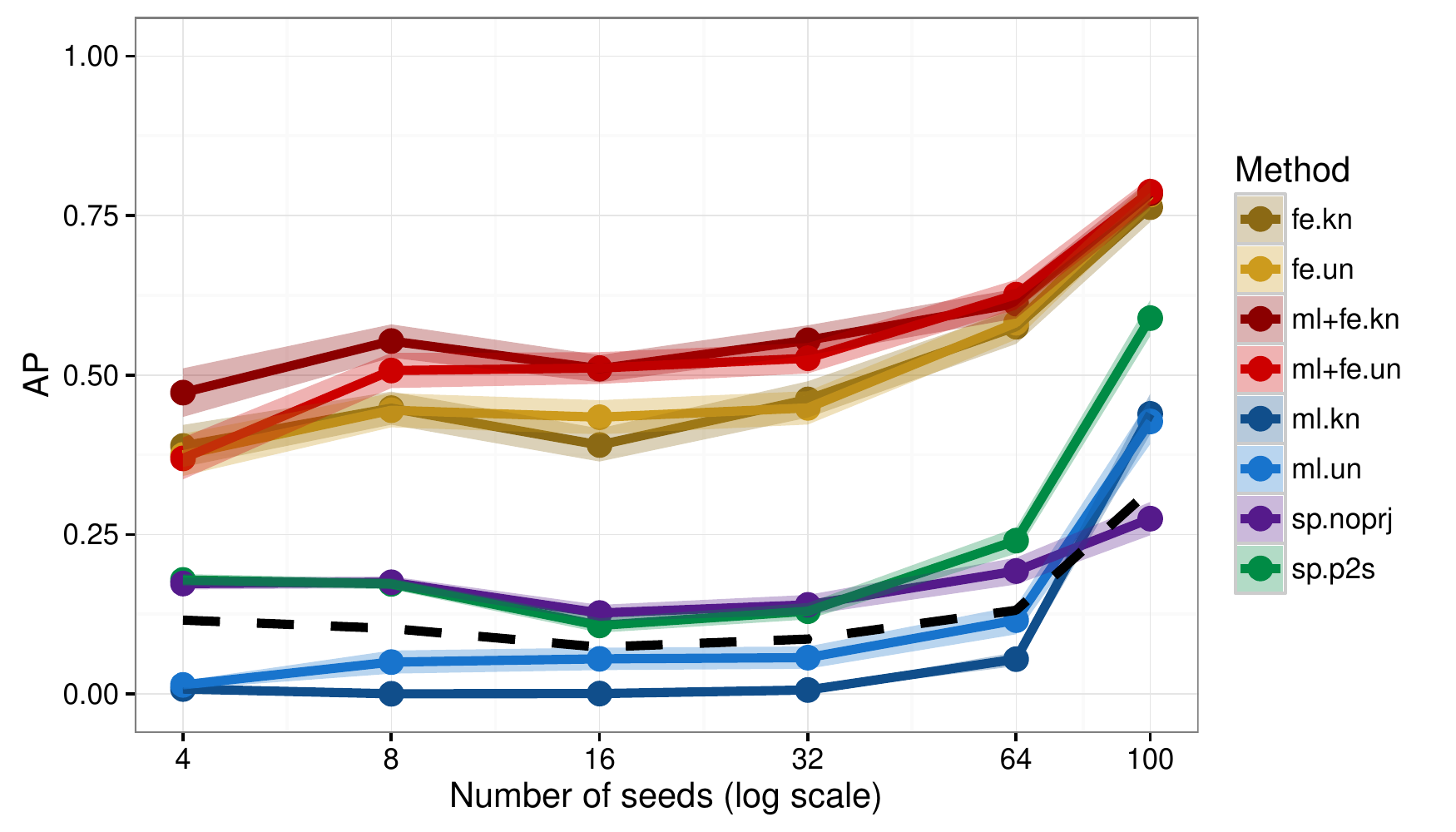} }
  \vspace{-0.2cm}
  \caption{ (a) The adjacency matrix of the Florida Bay trophic network.
        Nodes correspond to classes of plants and animals
        (e.g., sharks, rays, shorebirds, zooplankton, phytoplankton).
        An edge joins two nodes if the corresponding organisms are in
        a predator-prey relation. The sixteen types of birds in the
        network are highlighted in the red block.
        Note the disassortative structure of the bird block
        (the edges within the red block are all incident to the node that
        corresponds to raptors).
\
        (b) Average precision in identifying the bird nodes as a function
        of the number of seed vertices for
        ML vertex nomination (blue),
        restricted-focus ML (red),
        and spectral vertex nomination with (green)
        and without projection to the sphere (violet),
        when model parameters are
        known (light colors) and unknown (dark colors).
        The black dashed line indicates chance performance.}
  \label{fig:florida}
\end{figure}

\section{Discussion and Future Work}
\label{sec:discussion}

Network data has become ubiquitous in the sciences,
giving rise to a vast array of computational and statistical problems
that are only beginning to be explored.
In this paper, we have explored one such problem that arises when working
with network data, namely the task of performing vertex nomination.
This task, in some sense the graph analogue of the
classic information retrieval problem, is fundamental to
exploratory data analysis on graphs as well as to machine learning
applications.
Above, we established the consistency of two methods of vertex nomination:
a maximum-likelihood scheme $\Lml$ and
its restricted-focus variant $\LmlRES$, in which
we obtain a feasibly exactly-solvable optimization problem at the expense
of using less than the full information available in the graph.
Additionally, we have introduced a maximum-likelihood vertex nomination scheme
for the case where vertices are endowed with features
and when (possibly weighted) edges are drawn from a canonical exponential
family.
The key to all of these methods is the ability to quickly approximate a
solution to the seeded graph matching problem.

We have presented experimental comparisons of these methods against each other
and against several other benchmark methods, where we see that
the best choice of method depends highly on graph
size and structure. The major tradeoff appears to be that large graphs
(tens of thousands of vertices) are not tractable for $\Lml$,
but in smaller and medium-sized graphs, $\Lml$
can detect signal where spectral methods fail to do so.
It is worth noting that $\Lml$, and, to a lesser extent, $\LmlRES$,
is quite competitive with $\Lspec$, and even manages to best $\Lspec$
when the structure of the graph is ill-suited to the typical assumptions
of spectral methods, as in the case of our linguistic data set.
All told, our experimental results mirror those
in~\cite{fishkind2015vertex} and point toward a theory of which methods are
best-suited to which graphs, a direction that warrants further exploration.





\appendix
\section{Proof Details}
Before proving Theorem~\ref{thm:mlconsis}, we first state a useful
initial proposition.
\begin{proposition}
\label{prop:cycle}
Let $\vec x=(x_1,x_2,\ldots,x_k)$ be a vector with distinct entries in $\mathbb{R}^{k}$.  Let $f(\cdot)$ be a strictly increasing real valued function (with the abuse of notation, $f(\vec x$), denoting $f(\cdot)$ applied entrywise to $\vec x$).
Let the order statistics of $\vec x$ be denoted
$$x_{(1)}< x_{(2)}<\cdots< x_{(k)},$$
and define
$\alpha=\min_{i\in\{2,3,\ldots,k\}} |x_{(i)}-x_{(i-1)}|$, and $\beta=\min_{i\in\{2,3,\ldots,k\}} |f(x_{(i)})-f(x_{(i-1)})|$.
If $\sigma$ is the cyclic permutation
$$\sigma=\begin{pmatrix}
  1 & 2 & 3 &\cdots & k \\
  2 & 3 & 4& \cdots & 1 
 \end{pmatrix},$$
then
$$\langle \vec x,f(\vec x)\rangle -\langle \vec x,f(\sigma(\vec x))\rangle\geq (k-1)\alpha\beta.$$
\end{proposition}
\begin{proof}
We will induct on $k$.  To establish the base case, $k=2$, let $x_1=x_{(1)}$
without loss of generality and observe that
\begin{align*}
\langle \vec x,f(\vec x)\rangle -\langle \vec x,f(\sigma(\vec x))\rangle&=(x_2-x_1)(f(x_2)-f(x_1))\\
&=(x_{(2)}-x_{(1)})(f(x_{(2)})-f(x_{(1)}))\geq \alpha\beta.
\end{align*}
For general $k$, again, without loss of generality let $x_1=x_{(1)}$, and define the permutation
$$\tau=\begin{pmatrix}
   2 & 3 &\cdots & k \\
   3 & 4& \cdots & 2 
 \end{pmatrix}.$$
 Then
 \begin{align*}
\langle \vec x,f(\vec x)\rangle -\langle \vec x,f(\sigma(\vec x))\rangle&=\langle \vec x,f(\vec x)\rangle -\langle \vec x,f(\tau(\vec x))\rangle+\langle \vec x,f(\tau(\vec x))\rangle -\langle \vec x,f(\sigma(\vec x))\rangle\\
&=\langle \vec x,f(\vec x)\rangle -\langle \vec x,f(\tau(\vec x))\rangle+(x_k-x_1)(f(x_2)-f(x_1))\\
&\geq \langle \vec x,f(\vec x)\rangle -\langle \vec x,f(\tau(\vec x))\rangle+\alpha\beta,
\end{align*}
and the result follows from the inductive hypothesis.
\end{proof}

\begin{remark} \label{rem:notallequal}
\emph{It follows immediately that in Proposition \ref{prop:cycle}, if there exists an index $i\in[k]$ such that
$\alpha_i=\min_{j\neq i} |x_{(i)}-x_{(j)}|>0$, and $\beta_i=\min_{j\neq i} |f(x_{(i)})-f(x_{(j)})|>0$, then
$\langle \vec x,f(\vec x)\rangle -\langle \vec x,f(\sigma(\vec x))\rangle\geq\alpha_i\beta_i.$}
\end{remark}

We are now ready to prove Theorem~\ref{thm:mlconsis}.
\begin{proof}[Proof of Theorem~\ref{thm:mlconsis}]
Define
$$X_P:=\tr(AB^\top)-\tr(A(I_m\oplus P)B (I_m\oplus P)^\top)$$
and define $\mathcal{P} = \{ P \in \Pi_{\fu} : \epsilonout{1}(P)>0\}$.
We will show that
$$ \p\left( \exists\,\, P \in \mathcal{P}\text{ s.t. } 
X_P \le 0 \right)
        = O(1/n^2), $$
from which the desired consistency of $\mathcal{L}^{\ML}$ follows
by the Borel-Cantelli Lemma, since this probability is summable in $n$.
Fix $P \in \mathcal{P}$,
and let $\sigma_P \in S_n$ be the permutation associated with  $I_m\oplus P$.
The action of shuffling $B$ via $I_m\oplus P$ is equivalent to permuting the $[n^2]$ elements of $\text{vec}(B)$ via a permutation $\tau_P$, in that
$$\tr(A(I_m\oplus P)B(I_m\oplus P)^\top)=\langle \text{vec}(A),\tau_P(\text{vec}(B))\rangle  .$$
Moreover,
$\tau_P$ can be chosen so that, in the cyclic decomposition of $\tau_P=\tau_P^{(1)}\tau_P^{(2)}\cdots\tau_P^{(\ell)}$, each (disjoint) cycle is acting on a set of distinct real numbers.
Note that Proposition \ref{prop:cycle} implies that the contribution of each cycle $\tau_P^{(i)}$ to $\e(X_P)$ is nonnegative, and the assumptions of Theorem \ref{thm:mlconsis} imply that for each $i,j\in[K]$ such that $i\neq j$, the contribution of each (nontrivial) cycle permuting a $\Lambda_{i,i}$ entry to a $\Lambda_{i,j}$ entry contributes at least $\alpha\beta$ to $\e(X_P)$.
It follows immediately that
\begin{align*}
\e(X_P)&=\e\left(\tr(AB)-\tr(APBP^\top)\right)\\
&=\e\left(\langle \text{vec}(A),\text{vec}(B)\rangle-\langle \text{vec}(A),\tau_P(\text{vec}(B)\rangle\right)\\
&\geq 2\alpha\beta\sum_i\left(\frac{1}{2}\sum_{j} \sum_{k \neq j} \epsilon_{i,j} \epsilon_{i,k}+m_i\epsilonout{i}\right)\\
        &\geq 2\alpha\beta\sum_i\left(\frac{(\fu_i - \epsilonout{i}) \epsilonout{i}}{2}+m_i\epsilonout{i}\right).
\end{align*}
Let $\mathfrak{n}(P)$ be the total number of distinct entries of vec$(B)$ permuted by $\tau_P$, and note that an application of Proposition \ref{prop:cycle} yields
\begin{align*}
\e(X_P)&=\e\left(\tr(AB)-\tr(APBP^\top)\right)\\
&=\e\left(\langle \text{vec}(A),\text{vec}(B)\rangle-\langle \text{vec}(A),\tau_P(\text{vec}(B)\rangle\right)\\
&\geq \frac{1}{2}\mathfrak{n}(P)\gamma\kappa.
\end{align*}
The assumptions in the Theorem also immediately yield that
$$\mathfrak{n}(P)\geq \sum_k\left(\frac{(\fu_k - \epsilonout{k}) \epsilonout{k}}{2}+m_k\epsilonout{k}\right).$$
We next note that $X_P$ is a sum of $\mathfrak{n}(P)$ independent random variables, each bounded in $[-c,c]$.
An application of Hoeffding's inequality then yields
\begin{equation*}
\begin{aligned}
\p(X_P \le 0 ) &\le \p\left( |X_P - \e X_P| \ge \e X_P \right)
\le 2\exp\left\{ -\frac{ 2 \e^2 X_P }{ 4c^2 \mathfrak{n}(P) } \right\} \\
&\le 2\exp\left\{ -\frac{  |\e X_P| \kappa\gamma }{ 2c^2 } \right\}
\le
  2\exp\left\{
        -\frac{  \alpha\beta \kappa\gamma }{ 4c^2 }
        \sum_k\left(\frac{(\fu_k - \epsilonout{k}) \epsilonout{k}}{2}+m_k\epsilonout{k}\right)
  \right\} .
\end{aligned}
\end{equation*}

Next, note that
$$|\{P\in\mathcal{P}\text{ s.t. } 
X_P\leq 0\}|=0\text{ iff
 }|\{P\in\mathcal{P}/\sim\text{ s.t. }
 X_P\leq 0\}|=0.$$
Given $\{\epsilon_{k,\ell}\}_{k,\ell=1}^K$ satisfying $\fu_k=\sum_{\ell}\epsilon_{k,\ell}=\sum_{\ell}\epsilon_{\ell,k}$ for all $k\in[K]$, the number of elements $P\in \mathcal{P}/\sim$ with
$\epsilon_{k,\ell}(P)=\epsilon_{k,\ell}$ for all $k,\ell \in[K]$ is at most
\begin{align}
\label{bnd2}
\fu_1^{\sum_{\ell\neq 1}\epsilon_{1,\ell}}\fu_2^{\sum_{\ell \neq 2}\epsilon_{2, \ell }}\cdots \fu_K^{\sum_{\ell \neq K}\epsilon_{K,\ell}}
&=\fu_1^{\fu_1-\epsilon_{1,1}}\fu_2^{\fu_2-\epsilon_{2,2}}\cdots \fu_K^{\fu_K-\epsilon_{K,K}}\notag\\
&=e^{\sum_k(\fu_k-\epsilon_{k,k})\log(\fu_k)}.
\end{align}
The number of ways to choose such a set (i.e. the $\{\epsilon_{k,\ell}\}_{k,\ell}^K$)
is bounded above by
\begin{equation}
\label{bnd3}
\prod_{\substack{k\text{ s.t. } \epsilonout{k}\neq0 }}(\fu_k+K)^K
=e^{\sum_{\substack{k\text{ s.t. } \epsilonout{k}\neq0 }} K\log(\fu_k+K)}.
\end{equation}
Applying the union bound over all $P \in \mathcal{P}/\sim$, we then have
\begin{align}
\label{eq:bnd3}
&\p  \big(\,\exists\, P \in \mathcal{P} \text{ s.t. }
X_P \le 0 \,\big)=\p  \big(\,\exists\, P \in \mathcal{P}/\sim \text{ s.t. }
X_P \le 0 \,\big) \notag\\
&\le \exp\bigg\{        
        -\frac{  \alpha\beta \kappa\gamma }{ 2c^2 }
        \sum_k\left(\frac{(\fu_k - \epsilonout{k}) \epsilonout{k}}{2}+m_k\epsilonout{k}\right) \\
        &\hspace{16mm}+ \sum_k (\fu_k - \epsilon_{k,k}) \log \fu_k
        + \!\!\!\!\!\sum_{\substack{k\text{ s.t. } \epsilonout{k}\neq0 }} \!\!\!\!\!\!K \log( \fu_k + K ) \bigg\} .
\end{align}
It remains for us to establish that the expression inside the
exponent goes to $-\infty$ fast enough to ensure our desired bound.
For each $k$, the contribution to the exponent in (\ref{eq:bnd3}) is
\begin{align}
\label{eq:eachi}
&-\frac{  \alpha\beta \kappa\gamma }{ 2c^2 }
        \left(\frac{(\fu_k - \epsilonout{k}) \epsilonout{k}}{2}+m_k\epsilonout{k}\right)
        + (\fu_k - \epsilon_{k,k}) \log \fu_k
        + \mathbbm{1}\{\epsilonout{k}\neq0\}  K \log( \fu_k + K )\notag\\
        &=-\frac{  \alpha\beta \kappa\gamma }{ 2c^2 }
        \left(\frac{\epsilon_{k,k} \epsilonout{k}}{2}+m_k\epsilonout{k}\right)
        + \epsilonout{k} \log \fu_k
        + \mathbbm{1}\{\epsilonout{k}\neq0\}  K \log( \fu_k + K )
\end{align}
If $\fu_k/2\leq\epsilon_{k,k}<\fu_k$, then
$$\epsilon_{k,k} \epsilonout{k}\geq\frac{\fu_k \epsilonout{k}}{2}=\omega(\epsilonout{k} \log \fu_k), \text{ and }\epsilon_{k,k} \epsilonout{k}\geq\frac{\fu_k \epsilonout{k}}{2}=\omega(K \log (\fu_k+K)),$$
and the contribution to the exponent in (\ref{eq:bnd3}) from $k$, Eq.~(\ref{eq:eachi}), is clearly bounded above by $-2\log(n)$ for sufficiently large $n$.
If $\epsilon_{k,k}\leq\fu_k/2$ then $\epsilonout{k}>\fu_k/2$, and
$$m_k \epsilonout{k}=\omega(\epsilonout{k} \log \fu_k), \text{ and }m_k \epsilonout{k}\geq\frac{m_k\fu_k}{2}=\omega(K \log (\fu_k+K)),$$
and the contribution to the exponent in (\ref{eq:bnd3}) from $k$, Eq.~(\ref{eq:eachi}), is clearly bounded above by $-2\log(n)$ for sufficiently large $n$.
If $\epsilon_{k,k}=\fu_k$, then all terms in the exponent (\ref{eq:eachi}) are equal to $0$.
For sufficiently large $n$, Eq.~(\ref{eq:bnd3}) is then bounded above by
\begin{align*}
\exp \left\{-\!\!\!\!\!\sum_{\substack{k\text{ s.t. } \epsilonout{k}\neq0 }} \!\!\!\!\!\!2\log(n)  \right\}\leq \exp \left\{-2\log(n)  \right\},
\end{align*}
and the result follows.
\end{proof}

Consistency of $\LmlRES$ as claimed in Theorem~\ref{thm:restmlconsis} follows similarly to that of $\Lml$, and we next briefly sketch the details of the proof.

\begin{proof}[Proof of Theorem~\ref{thm:restmlconsis} (Sketch)]
Analogously to the proof of Theorem~\ref{thm:mlconsis}, define
$$X_P:=\tr\left((A^{(1,2)})^\top B^{(1,2)} \right)-\tr\left((A^{(1,2)})^\top B^{(1,2)}P^\top\right).$$
The proof follows {\it mutatis mutandis} to the proof of Theorem~\ref{thm:mlconsis}, with the key difference being that in this case,
\begin{align*}
\e(X_P)&=\e\left(\tr\left((A^{(1,2)})^\top B^{(1,2)} \right)-\tr\left((A^{(1,2)})^\top B^{(1,2)}P^\top \right)\right)\\
&\geq 2\alpha\beta\sum_k m_k\epsilonout{k}.
\end{align*}
Details are omitted for brevity.
\end{proof}

Before proving Theorem~\ref{thm:hatconsistency}
we establish some preliminary concentration results for our estimates $\widehat \Lambda$, and $\hat n_k$, $k\in[K]$.
An application of
Hoeffding's inequality yields that for $k,\ell \in[K]$ such that $k\neq \ell$,
\begin{align}
\label{eq:concen1}
\p\left( \left| \widehat\Lambda_{k,\ell} -\Lambda_{k,\ell}\right|\geq \frac{\sqrt{n\log n}}{m_km_\ell} \right)\leq 2\text{exp}\left\{ -2n\log n\right\},
\end{align}
and for $k\in[K]$,
\begin{align}
\label{eq:concen2}
\p\left(  \left| \widehat\Lambda_{k,k} -\Lambda_{k,k} \right|\geq \frac{\sqrt{n\log n}}{\binom{m_k}{2}} \right)\leq 2\text{exp}\left\{-2n\log n \right\},
\end{align}
and
\begin{align}
\label{eq:concen3}
\p\left(  \left| \nhat_k -n_k \right|\geq t \right)\leq 2\text{exp}\left\{\frac{-2mt^2}{n^2} \right\},
\end{align}
With $\gamma$ defined as in (\ref{eq:gk}), define the events $\calEone_n$ and $\calEtwo_n$ via
$$
\calEone_n =\left\{\forall\,\, \{k,\ell\}\in\binom{[K]}{2},\text{ s.t }|\Lambda_{k,k}-\Lambda_{k,\ell}|>\gamma,\text{ it holds that } \left|\widehat\Lambda_{k,k}-\widehat\Lambda_{k,\ell}\right|>\frac{\gamma}{2}
\right\};
$$
$$
\calEtwo_n=\left\{\forall\,\, k\in[K],\, |\nhat_k-n_k|\leq n_k^{2/3}
\right\}.
$$
Combining (\ref{eq:concen1})--(\ref{eq:concen3}), we see that if for each $k\in[K]$, $n_k=\Theta(n)$,
$\min_k m_k=\omega(\sqrt{n_k}\log(n_k))$,
then for sufficiently large $n$,
\begin{align}
\label{eq:e1e2}
\p\left((\calEone_n\cup\calEtwo_n)^c\right)\leq e^{-2\log n} .
\end{align}

We are now ready to prove Theorem~\ref{thm:hatconsistency},
proving the consistency of $\Lml$ when the model parameters are unknown.

\begin{proof}[Proof of Theorem~\ref{thm:hatconsistency}]
Let $\widehat B$ be our estimate of $B$ using the seed vertices; i.e., there are $\nhat_k$ vertices from block $k$ for each $k\in[K]$, and for each $k,\ell \in[K]$, the entry of $\widehat B$ between a block $k$ vertex and a block $\ell$ vertex is
$$\log\left(  \frac{ \widehat\Lambda_{k,\ell} }{1-\widehat \Lambda_{k,\ell} }\right).$$
Let $\widehat L$ be the set of distinct entries of $\widehat \Lambda$, and define
\begin{align}
\label{eq:abcest}
&\hat\alpha = \min_{\{k,\ell\}\text{ s.t. } k\neq \ell }|\widehat \Lambda_{k,k} - \widehat \Lambda_{k,\ell}|
\hspace{5mm}
\hat\beta = \min_{\{k,\ell\}\text{ s.t. } k\neq \ell }|\widehat B_{k,k} - B_{k,\ell}|
\hspace{5mm}\hat c=\max_{i,j,k,\ell}|\widehat B_{i,j}-\widehat B_{k,\ell}|,\\
\label{eq:gkest}
&\hspace{20mm}\hat\gamma = \min_{x,y\in \widehat L}|x-y|,\hspace{5mm}\hat\kappa = \min_{x,y\in\widehat  L}\left|\log\left(\frac{x}{1-x}\right)-\log\left(\frac{y}{1-y}\right)\right|.
\end{align}
Note that conditioning on $\calEone_n\cup\calEtwo_n$ and assumption {\it iv.} ensures that each of $\hat\alpha$, $\hat\beta$, $\hat c$, $\hat\gamma$, and $\hat\kappa$ is bounded away from $0$ by an absolute constant for sufficiently large $n$.
For each $k\in[K],$ define
\begin{align}
\label{eq:eest}
&\mathfrak{e}_k:=|\nhat_k -n_k|=|\mathfrak{\hat u}_k-\fu_k|, \hspace{5mm} \mathfrak{e}=\sum_k\mathfrak{e}_k,\hspace{5mm}\eta_k:=\min(n_k,\nhat_k),\hspace{5mm} \eta=\sum_k\eta_k,
\end{align}
and note that conditioning on $\calEone_n\cup\calEtwo_n$ ensures that $\fe_k=O(n_k^{2/3})$ for all $k\in[K].$
An immediate result of this is that, conditioning on $\calEone_n\cup\calEtwo_n$, we have that $\eta_k=\Theta(n_k)=\Theta(n)$ for all $k\in[K]$.

Define $\mathcal{P} := \{ P \in \Pi_{\fu} : \epsilonout{1}(P)>n^{2/3}\log n\}$,
and for $P\in\Pi_\fu$, define
$$X_P:=\tr(A\tilde B^\top)-\tr(A(I_m\oplus P)\tilde B (I_m\oplus P)^\top).$$
We will
show that
$$ \p\left( \exists\,\, P \in \mathcal{P}\text{ s.t. } 
X_P \le 0 \right)
        = O(1/n^2), $$
and the desired consistency of $\Lml$ follows immediately.
To this end, decompose $A$ and $B$ as
\[
  A = \kbordermatrix{
    & \eta & \mathfrak{e}  \\
    \eta & A^{(c,c)} & A^{(c,e)} \\
    \fe & A^{(e,c)} & A^{(e,e)}
  }\hspace {10mm}   B = \kbordermatrix{
    & \eta & \fe  \\
    \eta & B^{(c,c)} & B^{(c,e)} \\
    \fe & B^{(e,c)} & B^{(e,e)}
  },
\]
where $A^{(c,c)}$ (resp., $B^{(c,c)}$) is an $\eta\times\eta$ submatrix of $A$ (resp., $B$)---which contains the seed vertices in $A$---with exactly $\eta_k$ vertices (resp., labels) from block $k$ for each $k\in[K]$.
We view $A^{(c,c)}$ as the ``core'' matrix of $A$ (with $A^{(e,e)}$ and $A^{(c,e)}$ being the ``errorful'' part of $A$), as $A^{(c,c)}$ is a submatrix of $A$ that we could potentially cluster perfectly along block assignments.
Note that similarly decomposing $P$ as
\[
  P= \kbordermatrix{
    & \eta & \mathfrak{e}  \\
    \eta & P^{(c,c)} & P^{(c,e)} \\
    \fe & P^{(e,c)} & P^{(e,e)}
  },
\]
we see that there exists a principal permutation submatrix of $P^{(c,c)}$ of size $(\eta-2\fe)\times(\eta-2\fe)$, which we denote $\tilde P$ (with associated permutation $\tilde\sigma$).
This matrix represents a subgraph of the core vertices of $A$ mapped to a subgraph of the core vertices in $B$.
We can then write $P=\tilde P\oplus Q,$ where $Q\in\Pi_{3\fe}.$
For each $k,\ell \in[K],$ let
$$\tilde\epsilon_{k,\ell}=\tilde\epsilon_{k,\ell}(\tilde P)=|\{v\in U_k\text{ s.t. }\tilde\sigma(v)\in U_k   \}|$$
Consider now
\begin{align}
\label{eq:Xpest}
X_P&=\tr(A(I_{\eta-3\fe}\oplus Q)B(I_{\eta-3\fe}\oplus Q)^\top)-\tr(A(\tilde P\oplus Q)B(\tilde P\oplus Q)^\top).
\end{align}
Letting $\mathfrak{\tilde u}_k$ denote the number of vertices from
the $k$-th block acted on by $\tilde P$, our assumptions yield
\begin{align*}
\e(X_P)\geq 2\hat\alpha\hat\beta\sum_k\left(\frac{(\mathfrak{\tilde u}_k - \tepsilonout{k}) \tepsilonout{k}}{2}+m_k\tepsilonout{k}\right)-\Theta(\eta\fe)-\Theta(\fe^2).
\end{align*}
Let $\mathfrak{\tilde n}(P)$ be the total number of distinct entries of vec$(B^{(c,c)})$ permuted by $\tilde P$, and note that another application of Proposition \ref{prop:cycle} yields
\begin{align*}
\e(X_P)\geq \frac{1}{2}\mathfrak{\tilde n}(P)\hat\gamma\hat\kappa-\Theta(\eta\fe)-\Theta(\fe^2).
\end{align*}
The assumptions in the Theorem also immediately yield that
$$\mathfrak{\tilde n}(P)\geq \sum_k\left(\frac{(\tfu_k - \tepsilonout{k}) \tepsilonout{k}}{2}+m_k\tepsilonout{k}\right).$$
We then have that there exists a constants $c_1>0$ and $c_2>0$ such that
\begin{align}
\label{eq:bnd4}
&\p  \big(\,\exists\, P \in \mathcal{P} \text{ s.t. }
X_P \le 0 \,\big|\,\calEone_n\cup\calEtwo_n\big)=\p  \big(\,\exists\, P \in \mathcal{P}/\sim \text{ s.t. }
X_P \le 0 \,\big|\, \calEone_n\cup\calEtwo_n\big) \notag\\
&\le \exp\bigg\{        
        -\frac{  \hat \alpha\hat \beta \hat \kappa\hat \gamma }{ 2\hat c^2 }
        \sum_k\left(\frac{(\tfu_k - \tepsilonout{k}) \tepsilonout{k}}{2}+m_k\tepsilonout{k}\right)
        + \Theta(\eta\fe)+\Theta(\fe^2) \notag\\
        &\hspace{16mm}+\sum_k (\tfu_k -\tilde\epsilon_{k,k}) \log \tfu_k
        +~~~~\sum_{\substack{k\text{ s.t. } \tepsilonout{k}\neq0 }} \!\!\!\!\!\!K \log( \tfu_k + K )+O(\fe\log\fe)\bigg\} \notag\\
         &=\exp\bigg\{   
        -c_1
        \sum_k\left(\frac{(\tfu_k - \tepsilonout{k}) \tepsilonout{k}}{2}+m_k\tepsilonout{k}\right) \\
         &\hspace{16mm}+\sum_k \tepsilonout{k} \log \tfu_k
        + \!\!\!\!\sum_{\substack{k\text{ s.t. } \tepsilonout{k}\neq0 }} \!\!\!\!\!\!K \log( \tfu_k + K )+ \Theta(n\fe)\bigg\}\notag\\
        &\leq\text{exp}\{-c_2 n^{7/4}\log n \}.
\end{align}
Unconditioning Equation~\eqref{eq:bnd4} combined with Equation~\eqref{eq:e1e2} yields the desired result.
\end{proof}

\begin{proof}[Proof of Theorem~\ref{thm:rhatconsistency} (Sketch)]
The proof of Theorem~\ref{thm:rhatconsistency} is a straightforward combination of the proofs of Theorems \ref{thm:restmlconsis} and \ref{thm:hatconsistency} once we have defined
$$\mathcal{P} := \{ P \in \Pi_{\fu} : \epsilonout{1}(P)>n^{8/9}\log n\}. $$
Details are omitted for the sake of brevity.
\end{proof}

\vskip 0.2in
\bibliographystyle{plainnat}
\bibliography{VN_ML}

\begin{thebibliography}{42}
\providecommand{\natexlab}[1]{#1}
\providecommand{\url}[1]{\texttt{#1}}
\expandafter\ifx\csname urlstyle\endcsname\relax
  \providecommand{\doi}[1]{doi: #1}\else
  \providecommand{\doi}{doi: \begingroup \urlstyle{rm}\Url}\fi

\bibitem[Adamic and Glance(2005)]{AdaGla2005}
L.~A. Adamic and N.~Glance.
\newblock The political blogosphere and the 2004 {US} election.
\newblock In \emph{Proc. {WWW}-2005 Workshop on the Weblogging Ecosystem},
  2005.

\bibitem[Airoldi et~al.(2008)Airoldi, Blei, Fienberg, and Xing]{Airoldi2008}
E.~M. Airoldi, D.~M. Blei, S.~E. Fienberg, and E.~P. Xing.
\newblock {Mixed membership stochastic blockmodels}.
\newblock \emph{The Journal of Machine Learning Research}, 9:\penalty0
  1981--2014, 2008.

\bibitem[Airoldi et~al.(2013)Airoldi, Costa, and Chan]{airoldi13:_stoch}
E.~M. Airoldi, T.~B. Costa, and S.~H. Chan.
\newblock Stochastic blockmodel approximation of a graphon: Theory and
  consistent estimation.
\newblock \emph{Advances in Neural Information Processing Systems},
  26:\penalty0 692--700, 2013.

\bibitem[Bickel and Chen(2009)]{BicChe2009}
P.~J. Bickel and A.~Chen.
\newblock A nonparametric view of network models and {Newman-Girvan} and other
  modularities.
\newblock \emph{Proc. National Academy of Sciences, USA}, 106:\penalty0
  21068--21073, 2009.

\bibitem[Bullmore and Sporns(2009)]{bullmore2009complex}
E.~Bullmore and O.~Sporns.
\newblock Complex brain networks: graph theoretical analysis of structural and
  functional systems.
\newblock \emph{Nature Reviews Neuroscience}, 10\penalty0 (3):\penalty0
  186--198, 2009.

\bibitem[Carrington et~al.(2005)Carrington, Scott, and
  Wasserman]{carrington2005models}
P.~J. Carrington, J.~Scott, and S.~Wasserman.
\newblock \emph{Models and Methods in Social Network Analysis}.
\newblock Cambridge University Press, 2005.

\bibitem[Coppersmith(2014)]{coppersmith2014vertex}
G.~Coppersmith.
\newblock Vertex nomination.
\newblock \emph{Wiley Interdisciplinary Reviews: Computational Statistics},
  6\penalty0 (2):\penalty0 144--153, 2014.

\bibitem[Coppersmith and Priebe(2012)]{coppersmith2012vertex}
G.~A. Coppersmith and C.~E. Priebe.
\newblock Vertex nomination via content and context.
\newblock \emph{arXiv preprint arXiv:1201.4118}, 2012.

\bibitem[Fiori et~al.(2013)Fiori, Sprechmann, Vogelstein, Musé, and
  Sapiro]{jovo}
M.~Fiori, P.~Sprechmann, J.~Vogelstein, P.~Musé, and G.~Sapiro.
\newblock Robust multimodal graph matching: Sparse coding meets graph matching.
\newblock \emph{Advances in Neural Information Processing Systems}, pages
  127--135, 2013.

\bibitem[Fishkind et~al.(2015)Fishkind, Lyzinski, Pao, Chen, and
  Priebe]{fishkind2015vertex}
D.~E. Fishkind, V.~Lyzinski, H.~Pao, L.~Chen, and C.~E. Priebe.
\newblock Vertex nomination schemes for membership prediction.
\newblock \emph{The Annals of Applied Statistics}, 9\penalty0 (3):\penalty0
  1510--1532, 2015.

\bibitem[Fishkind et~al.(2012)Fishkind, Adali, and Priebe]{FAP}
D.E. Fishkind, S.~Adali, and C.E. Priebe.
\newblock Seeded graph matching.
\newblock \emph{arXiv preprint arXiv:1209.0367}, 2012.

\bibitem[Foggia et~al.(2014)Foggia, Percannella, and Vento]{foggia2014graph}
P.~Foggia, G.~Percannella, and M.~Vento.
\newblock Graph matching and learning in pattern recognition in the last 10
  years.
\newblock \emph{International Journal of Pattern Recognition and Artificial
  Intelligence}, 28\penalty0 (01):\penalty0 1450001, 2014.

\bibitem[Franke and Wolfe(2016)]{FraWol2016}
B.~Franke and P.~J. Wolfe.
\newblock Network modularity in the presence of covariates.
\newblock \emph{arXiv preprint arXiv:1603.01214}, 2016.

\bibitem[Hubert and Arabie(1985)]{HubAra1985}
L.~Hubert and P.~Arabie.
\newblock Comparing partitions.
\newblock \emph{J. Classification}, 2:\penalty0 193--218, 1985.

\bibitem[Jeong et~al.(2001)Jeong, Mason, Barab{\'a}si, and
  Oltvai]{jeong2001lethality}
H.~Jeong, S.~P. Mason, A.-L. Barab{\'a}si, and Z.~N. Oltvai.
\newblock Lethality and centrality in protein networks.
\newblock \emph{Nature}, 411\penalty0 (6833):\penalty0 41--42, 2001.

\bibitem[Jonker and Volgenant(1987)]{jonker1987shortest}
R.~Jonker and A.~Volgenant.
\newblock A shortest augmenting path algorithm for dense and sparse linear
  assignment problems.
\newblock \emph{Computing}, 38\penalty0 (4):\penalty0 325--340, 1987.

\bibitem[Kandel et~al.(2007)Kandel, Bunke, and Last]{kandel2007applied}
A.~Kandel, H.~Bunke, and M.~Last.
\newblock \emph{Applied Graph Theory in Computer Vision and Pattern
  Recognition}, volume~1.
\newblock Springer, 2007.

\bibitem[Karrer and Newman(2011)]{karrer11:_stoch}
B.~Karrer and M.~E.~J. Newman.
\newblock Stochastic blockmodels and community structure in networks.
\newblock \emph{Physical Review E}, 83, 2011.

\bibitem[Kuhn(1955)]{hungarian}
H.~W. Kuhn.
\newblock The {H}ungarian method for the assignment problem.
\newblock \emph{Naval Research Logistic Quarterly}, 2:\penalty0 83--97, 1955.

\bibitem[Luxburg(2007)]{von2007tutorial}
U.~Von Luxburg.
\newblock A tutorial on spectral clustering.
\newblock \emph{Statistics and Computing}, 17\penalty0 (4):\penalty0 395--416,
  2007.

\bibitem[Lyzinski et~al.(2014{\natexlab{a}})Lyzinski, Fishkind, and
  Priebe]{JMLR:v15:lyzinski14a}
V.~Lyzinski, D.E. Fishkind, and C.E. Priebe.
\newblock Seeded graph matching for correlated {E}rd{\H os}-{R}\'{e}nyi graphs.
\newblock \emph{Journal of Machine Learning Research}, 15:\penalty0 3513--3540,
  2014{\natexlab{a}}.

\bibitem[Lyzinski et~al.(2014{\natexlab{b}})Lyzinski, Sussman, Tang, Athreya,
  and Priebe]{perfect}
V.~Lyzinski, D.~L. Sussman, M.~Tang, A.~Athreya, and C.~E. Priebe.
\newblock Perfect clustering for stochastic blockmodel graphs via adjacency
  spectral embedding.
\newblock \emph{Electronic Journal of Statistics}, 8:\penalty0 2905--2922,
  2014{\natexlab{b}}.

\bibitem[Marchette et~al.(2011)Marchette, Priebe, and
  Coppersmith]{marchette2011vertex}
D.~Marchette, C.~E. Priebe, and G.~Coppersmith.
\newblock Vertex nomination via attributed random dot product graphs.
\newblock In \emph{Proceedings of the 57th ISI World Statistics Congress},
  volume~6, page~16, 2011.

\bibitem[Newman(2006{\natexlab{a}})]{Newman2006}
M.~E.~J. Newman.
\newblock Finding community structure in networks using the eigenvectors of
  matrices.
\newblock \emph{Phys. Rev. E}, 74\penalty0 (3):\penalty0 036104,
  2006{\natexlab{a}}.

\bibitem[Newman(2006{\natexlab{b}})]{newman2006modularity}
M.~E.~J. Newman.
\newblock Modularity and community structure in networks.
\newblock \emph{Proceedings of the National Academy of Sciences}, 103\penalty0
  (23):\penalty0 8577--8582, 2006{\natexlab{b}}.

\bibitem[Newman(2016)]{Newman2016}
M.~E.~J. Newman.
\newblock Community detection in networks: Modularity optimization and maximum
  likelihood are equivalent.
\newblock \emph{arXiv preprint arXiv:1606.02319}, 2016.

\bibitem[Newman and Clauset(2016)]{NewCla2016}
M.~E.~J. Newman and A.~Clauset.
\newblock Structure and inference in annotated networks.
\newblock \emph{Nature Communications}, 7\penalty0 (11863), 2016.

\bibitem[Newman and Girvan(2004)]{Newman2004}
M.~E.~J. Newman and M.~Girvan.
\newblock {Finding and evaluating community structure in networks}.
\newblock \emph{Physical Review}, 69\penalty0 (2):\penalty0 1--15, February
  2004.
\newblock ISSN 1539-3755.

\bibitem[Nooy et~al.(2011)Nooy, Mrvar, and Batagelj]{DeNMrvBat2011}
W.~De Nooy, A.~Mrvar, and V.~Batagelj.
\newblock \emph{Exploratory social network analysis with {Pajek}}.
\newblock Cambridge University Press, 2011.

\bibitem[Olhede and Wolfe(2014)]{olhede_wolfe_histogram}
S.~C. Olhede and P.~J. Wolfe.
\newblock Network histograms and universality of block model approximation.
\newblock \emph{Proceedings of the National Academy of Sciences}, 111:\penalty0
  14722--14727, 2014.

\bibitem[Resnick and Varian(1997)]{resnick1997recommender}
P.~Resnick and H.~R. Varian.
\newblock Recommender systems.
\newblock \emph{Communications of the ACM}, 40\penalty0 (3):\penalty0 56--58,
  1997.

\bibitem[Ricci et~al.(2011)Ricci, Rokach, and Shapira]{ricci2011introduction}
F.~Ricci, L.~Rokach, and B.~Shapira.
\newblock \emph{Introduction to recommender systems handbook}.
\newblock Springer, 2011.

\bibitem[Rohe et~al.(2011)Rohe, Chatterjee, and Yu]{rohe2011spectral}
K.~Rohe, S.~Chatterjee, and B.~Yu.
\newblock Spectral clustering and the high-dimensional stochastic blockmodel.
\newblock \emph{Annals of Statistics}, 39:\penalty0 1878--1915, 2011.

\bibitem[Sussman et~al.(2012)Sussman, Tang, Fishkind, and Priebe]{STFP}
D.~L. Sussman, M.~Tang, D.~E. Fishkind, and C.~E. Priebe.
\newblock A consistent adjacency spectral embedding for stochastic blockmodel
  graphs.
\newblock \emph{Journal of the American Statistical Association}, 107\penalty0
  (499):\penalty0 1119--1128, 2012.

\bibitem[Ulanowicz et~al.(1997)Ulanowicz, Bondavalli, and
  Egnotovich]{UlaBonEgn1997}
R.~E. Ulanowicz, C.~Bondavalli, and M.~S. Egnotovich.
\newblock Network analysis of trophic dynamics in {South Florida} ecosystems,
  {FY 97}: The {Florida Bay} ecosystem.
\newblock Annual Report to the {U.S.} Geological Survey, Biological Resources
  Division. Ref. No. [UMCES]CBL 98-123, 1997.

\bibitem[{Vogelstein} et~al.(2014){Vogelstein}, {Conroy}, {Lyzinski},
  {Podrazik}, {Kratzer}, {Harley}, {Fishkind}, {Vogelstein}, and {Priebe}]{FAQ}
J.T. {Vogelstein}, J.M. {Conroy}, V.~{Lyzinski}, L.J. {Podrazik}, S.G.
  {Kratzer}, E.T. {Harley}, D.E. {Fishkind}, R.J. {Vogelstein}, and C.E.
  {Priebe}.
\newblock {Fast Approximate Quadratic Programming for Graph Matching}.
\newblock \emph{PLoS ONE}, 10\penalty0 (04), 2014.

\bibitem[Wasserman and Faust(1994)]{wasserman}
S.~Wasserman and K.~Faust.
\newblock \emph{Social Network Analysis: Methods and Applications}.
\newblock Cambridge University Press, 1994.

\bibitem[Yang and Leskovec(2015)]{yang2015defining}
J.~Yang and J.~Leskovec.
\newblock Defining and evaluating network communities based on ground-truth.
\newblock \emph{Knowledge and Information Systems}, 42\penalty0 (1):\penalty0
  181--213, 2015.

\bibitem[Yang et~al.(2013)Yang, McAuley, and Leskovec]{YanMcALes2013}
J.~Yang, J.~McAuley, and J.~Leskovec.
\newblock Community detection in networks with node attributes.
\newblock In \emph{Proc. {IEEE} 13th International Conference on Data Mining},
  pages 1151--1156, 2013.

\bibitem[Zachary(1977)]{Zachary1977}
W.~W. Zachary.
\newblock An information flow model for conflict and fission in small groups.
\newblock \emph{Journal of Anthropological Research}, 33\penalty0 (4):\penalty0
  452--473, 1977.

\bibitem[Zaslavskiy et~al.(2009)Zaslavskiy, Bach, and Vert]{Zaslavskiy2009}
M.~Zaslavskiy, F.~Bach, and J.P. Vert.
\newblock A path following algorithm for the graph matching problem.
\newblock \emph{IEEE Transactions on Pattern Analysis and Machine
  Intelligence}, 31\penalty0 (12):\penalty0 2227--2242, 2009.

\bibitem[Zhang et~al.(2015)Zhang, Levina, and Zhu]{ZhaLevZhu2015}
Y.~Zhang, E.~Levina, and J.~Zhu.
\newblock Community detection in networks with node features.
\newblock \emph{arXiv preprint arXiv:1509.01173}, 2015.

\end{thebibliography}

\end{document}